\documentclass{article}

\usepackage[main, final]{neurips_2026}
\usepackage{amsmath,amsfonts,bm}

\def\eqref#1{equation~\ref{#1}}
\def\1{\bm{1}}

\DeclareMathAlphabet{\mathsfit}{\encodingdefault}{\sfdefault}{m}{sl}
\SetMathAlphabet{\mathsfit}{bold}{\encodingdefault}{\sfdefault}{bx}{n}

\usepackage{xcolor}         % colors
\definecolor{mygreen}{HTML}{05A615}
\definecolor{myblue}{HTML}{598BE7}
\usepackage[utf8]{inputenc} % allow utf-8 input
\usepackage[T1]{fontenc}    % use 8-bit T1 fonts
\usepackage[colorlinks,
			linkcolor=myblue,
			anchorcolor=myblue,
			citecolor=myblue
			]{hyperref}      % hyperlinks
            
\usepackage{url}            % simple URL typesetting
\usepackage{booktabs}       % professional-quality tables
\usepackage{amsfonts}       % blackboard math symbols
\usepackage{nicefrac}       % compact symbols for 1/2, etc.
\usepackage{microtype}      % microtypography
\usepackage{colortbl}
\usepackage{caption}
\usepackage{wrapfig}
\usepackage{float} 
\usepackage{chngpage}
\usepackage{amsmath}
\usepackage{amssymb}
\usepackage{amsthm}
\usepackage{bm}
\usepackage{mathtools}
\usepackage{algorithm}
\usepackage{algorithmic}
\usepackage[font=small]{caption}
\usepackage{multirow}
\usepackage{adjustbox}
\usepackage{graphicx}
\usepackage{enumitem}
\usepackage[capitalize,noabbrev]{cleveref}

\usepackage{wrapfig}
\usepackage{dsfont}
\theoremstyle{plain}
\newtheorem{assumption}{Assumption}
\newtheorem{lemma}{Lemma}
\newtheorem{theorem}{Theorem}

\newtheorem*{theoremonefull}{Theorem 1}
\newtheorem*{theoremtwofull}{Theorem 2}

\theoremstyle{definition}

\theoremstyle{remark}

\usepackage[textsize=tiny]{todonotes}

\title{Higher-Order Action Supervision Makes A Strong Policy Class}

\author{
\textbf{Peng Cheng}$^{1,4}$\thanks{Equal contribution.},
Yunxian Hou$^{2}$\footnotemark[1],
Zhi Zhou$^{1,4}$,
\textbf{Qian Zhang}$^{3}$,
\textbf{Chang Huang}$^{3}$,
\textbf{Xianyuan Zhan}$^{1}$\thanks{Corresponding author.}\\
$^1$ Institute for AI Industry Research (AIR), Tsinghua University\\
$^2$ University of Electronic Science and Technology of China\\
$^3$ Horizon Robotics \quad $^4$ Zenfex AI Lab
}

\begin{document}

\maketitle
\setlength{\textfloatsep}{6pt}
\begin{abstract}
Modern data-driven decision-making methods, such as imitation learning (IL) and reinforcement learning (RL), have achieved great success in solving many complex tasks. However, these methods often suffer from serious control instability and robustness issues when applied in real-world applications such as robotics and autonomous driving, posing notable challenges for their practical deployment. We argue that this instability issue stems largely from their limitations in solely supervising and optimizing zeroth-order actions (i.e., the action labels), failing to account for higher-order action dynamics and temporal consistency. In this paper, we show that simultaneously supervising both zeroth- and first-order actions can dramatically enhance policies' performance and control robustness. To achieve this, we introduce a novel and elegant loss scheme supported by formal theoretical guarantees that can equip any off-the-shelf policy model (e.g., deterministic, stochastic, or flow policies) with the capability for higher-order action supervision, without requiring any structural modifications. Moreover, our proposed method can serve as a lightweight plug-and-play module that seamlessly integrates with a broad spectrum of existing offline RL frameworks. Extensive evaluations on OGBench and D4RL demonstrate that our approach yields substantial performance and robustness improvements across a wide range of continuous control environments. Notably, our method can also enhance policies' out-of-distribution (OOD) generalization capability in the challenging low-data regime, making it an ideal tool in tackling many real-world control problems.

% While offline reinforcement learning (RL) has achieved significant progress in solving complex decision-making problems, current methods suffer from control instability in realistic applications such as robotics and autonomous driving. We argue that this instability stems from the standard policy optimization process, which relies on a purely zeroth-order action matching paradigm. This absence of higher-order supervision restricts the learning process to point-wise action matching, thereby neglecting the local structural consistency of the policy. 
% In this study, we argue that relying solely on zeroth-order constraints is insufficient and that explicitly modeling first-order action variations is essential to ensure control stability. We therefore propose a simple yet effective policy class designed to simultaneously fit both zeroth-order action signals and first-order action variations, supported by formal theoretical guarantees. Furthermore, the proposed method functions as a lightweight plug-in module that readily integrates with mainstream offline RL frameworks. Extensive empirical evaluations demonstrate that our approach yields substantial performance improvements across various continuous control environments. Notably, in challenging limited-data regimes, this dual-constraint formulation provides reliable support for out-of-distribution (OOD) generalization.

\end{abstract}

\section{Introduction}
\label{introducion}
% Offline reinforcement learning (RL) aims to learn effective policies from static datasets without online interaction, making it crucial for real-world applications where data collection is expensive or risky \citep{levine2020offline, zhan2021deepthermal}. Existing policy optimization paradigms in offline RL encompass several technical frameworks, including deterministic modeling approaches \citep{fujimoto2021minimalist}, stochastic methods to extract high-reward actions \citep{kostrikov2021offline}, and score-based generative models designed to capture multimodal action distributions \citep{wang2022diffusion}.

Learning-based decision-making methods, such as imitation learning (IL) and reinforcement learning (RL), have achieved notable success in many complex tasks, such as games~\citep{mnih2015human, silver2017mastering}, robotics~\citep{kalashnikov2018qt,chi2023diffusion}, and autonomous driving~\citep{kendall2019learning, zheng2025diffusionbased}. However, when these learning-based approaches are deployed in real-world control problems, they are frequently criticized for their control instability and poor robustness. Examples of such issues include lack of smoothness in action outputs, severe jittering across consecutive steps, and poor generalization in offline learning under data sparsity~\citep{dulac2021challenges, song2023lipsnet,cheng2024look,cheng2026sample}.

In contrast to conventional control theories~\citep{kirk2004optimal} that place heavy emphasis on control stability and higher-order behavior modeling, contemporary learning-based decision-making methods in the AI community predominantly only focus on supervising and optimizing zeroth-order actions (i.e., the actual action labels).
In realistic control scenarios, the neglect of higher-order temporal consistency in actions frequently leads to control instability and non-smooth action execution, which constitutes a critical bottleneck for real-world deployment \citep{dulac2021challenges}. Under physical hardware constraints, such high-frequency action fluctuations can cause excessive mechanical wear and energy consumption, drastically shortening the lifespan of both the system and its constituent parts~\citep{song2023lipsnet,song2025lipsnet++,mysore2021regularizing,de2023oscillating,zhan2022deepthermal,zhan2025data}. These vulnerabilities pose severe safety risks and degrade operational reliability across various domains, including drone control \citep{mysore2021regularizing,shi2019neural}, robotic manipulation \citep{yu2021taac,zhang2024robust}, and autonomous driving \citep{cai2020high,chen2021addressing,wasala2020trajectory,lee2024gradient}. Fundamentally, this issue arises because the policy becomes excessively sensitive to minor state-action perturbations, yielding disproportionately large action variations \citep{song2023lipsnet,chen2021addressing}. We empirically illustrate this vulnerability using a maze navigation task in Figure \ref{fig:toy_case}. As shown in Figure \ref{fig:toy_case}(a), the policy learned via standard behavior cloning (BC), which lacks higher-order action supervision, produces large and highly oscillatory first-order actions, eventually producing a suboptimal trajectory. By contrast, jointly supervising both zeroth- and first-order actions enforces much stronger behavioral regularization, resulting in significantly smoother and optimal execution (Figure \ref{fig:toy_case}(b)). 
% However, a fundamental limitation shared by these approaches is their exclusive reliance on zeroth-order action supervision. Specifically, these methods focus on point-wise supervision to align the policy with the action labels in the dataset, while overlooking the higher-order structural information of the policy manifold. 
% irregular first-order actions, leading to highly unstable zero-order trajectories. Conversely, introducing higher-order supervision enforces consistent first-order actions, resulting in significantly smoother execution. Despite its practical significance, the challenge of ensuring stable execution through higher-order supervision remains under-explored in current offline RL literature.
\begin{figure}[t]
    \centering
    \includegraphics[width=\linewidth]{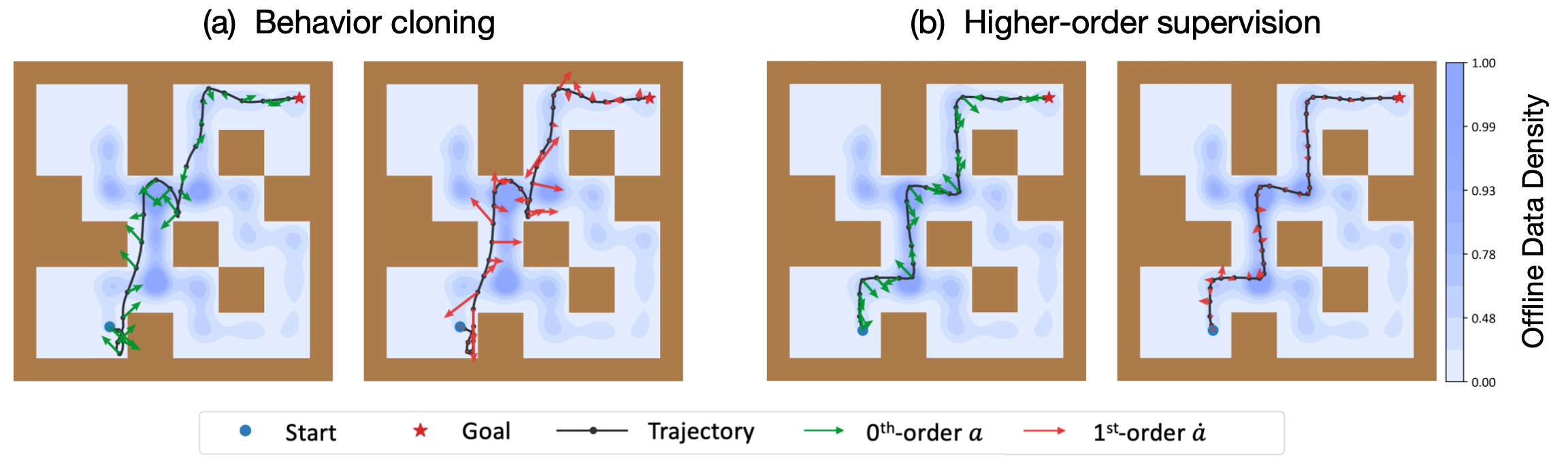}
    \caption{\small Trajectories and zeroth-/first-order actions of policies learned on the \texttt{maze2d} task with 20 successful trajectories and 500k training steps.
    Green arrows denote the zeroth-order action $a$ (joint torque), while red arrows denote the first-order action $\dot{a}$ (time derivative of joint torque, computed as $a'-a$).
    (a) Standard behavior cloning (BC) without $\dot{a}$ supervision, which yields a suboptimal trajectory with highly noisy and temporally inconsistent first-order actions.
    (b) Our method with higher-order action supervision, which yields an optimal navigation trajectory with much smoother and more stable zeroth- and first-order actions.
    }
    \label{fig:toy_case}
\end{figure}

However, incorporating higher-order action information in policy modeling is not an easy task. As the $0^\text{th}$-order action $a$ and the $1^\text{st}$-order action $\dot{a}$ live in two completely different action spaces, na\"ively predicting these two sets of actions not only doubles the action space, causing learning difficulties, but also faces potential conflicting gradient issues that are commonly encountered in multi-task learning~\citep{yu2020gradient}. Therefore, existing works on action smoothness improvement primarily focus on introducing action fluctuation penalty regularizations~\citep{mysore2021regularizing,kobayashi2022l2c2,shen2020deep,zhao2022adversarially}, or designing special neural network architectures to constrain the network's Lipschitz constant, e.g., via spectral normalization or specialized modules like LipsNet \citep{song2023lipsnet,song2025lipsnet++}. However, both streams of work suffer from a number of limitations. Adding smoothness penalties relies on task-specific ad-hoc design and cumbersome tuning efforts, whereas using specialized, heavily regularized neural network architectures loses generality and restricts the model's expressivity to capture complex data distributions.

In this work, we introduce a novel and elegant learning scheme that can achieve simultaneous $0^\text{th}$- and $1^\text{st}$-order action supervision. The best part is, we can easily transform any given policy model into a more powerful policy class without making any structural modification. The secret recipe is to leverage the intrinsic property of the policy model's time derivative for supervision, i.e., $\dot{\pi}(s)=\nabla_s \pi(s)\cdot \dot{s}$. We provide practical instantiations of loss schemes for deterministic, stochastic, and generative flow matching policy classes, supported by formal theoretical guarantees for the IL setting. Moreover, we show that our enhanced policy class is naturally compatible with the continuous-time RL paradigm, like the Hamilton-Jacobi-Bellman (HJB) framework. This leads to a plug-and-play enhancement module that can be seamlessly integrated with a wide spectrum of existing offline RL algorithms with minimal code modifications. We evaluate our method through extensive experiments on diverse tasks from the large-scale OGBench~\citep{ogbench_park2025} and widely used D4RL~\citep{fu2020d4rl} benchmarks. Despite its simplicity, our approach achieves consistent and significant performance improvements when added to different offline policy optimization paradigms. Notably, in challenging small-dataset settings, our proposed modification excels by providing robust OOD generalization support, enabling the agent to maintain stable and high performance even on tasks where the base algorithms fail or severely degrade. This makes our method an ideal tool for tackling many real-world control problems.

\section{Related work}
\label{sec:related_work}
\textbf{Smoothness and robustness enhancement for policy learning.}\quad Mitigating action fluctuation and trajectory drift remains a critical challenge for deploying policies in physical systems \citep{dulac2021challenges, song2023lipsnet}. To enforce action smoothness and robustness, prior studies in continuous control have primarily explored two directions: penalty-based regularizations~\citep{mysore2021regularizing,kobayashi2022l2c2,shen2020deep,zhao2022adversarially} and architectural enhancements~\citep{song2023lipsnet,song2025lipsnet++}. Penalty-based regularization approaches, such as CAPS \citep{mysore2021regularizing}, L2C2 \citep{kobayashi2022l2c2}, and SR$^2$L \citep{shen2020deep,zhao2022adversarially}, introduce action fluctuation penalty regularizations to their policy optimization objectives. However, they inherently rely on querying or sampling synthetic states to enforce local action consistency. In the offline setting, where interaction with the environment is not permitted, such state-dependent sampling may 
introduce significant OOD extrapolation errors, and often suffer from significant performance
variability across environments~\citep{christmann2024benchmarking}. In contrast, architectural enhancement methods explicitly constrain the network's Lipschitz constant, e.g., via spectral normalization or specialized modules like LipsNet \citep{song2023lipsnet,song2025lipsnet++}. However, imposing extra structural priors can also restrict the model's expressivity. This introduces a fundamental trade-off: while the policy becomes smoother, the restricted capacity commonly leads to underfitting, degraded control performance, and additional computational overhead during inference.
Concurrently, there are also a few attempts to incorporate higher-order action derivatives to enhance control smoothness, such as TaSIL~\citep{pfrommer2022tasil}, which minimizes the gap between the Taylor series expansion of the learned and expert policies. Although promising, it remains more theoretical rather than practical, due to the complex evaluation procedure involved.

% {\color{blue}
% In contrast, architectural methods that explicitly constrain the network's Lipschitz constant, e.g., via Spectral Normalization or specialized modules like LipsNet \citep{song2023lipsnet,song2025lipsnet++}, impose structural priors that can restrict the model's expressivity. Offline policies are learned entirely from static datasets, which often contain complex or multi-modal behavioral distributions; accurately modeling such data requires high-capacity networks. Furthermore, training these structurally constrained models effectively typically depends on extensive, high-quality data coverage, yet practical offline datasets are frequently limited in scale and consist predominantly of suboptimal transitions. This introduces a fundamental trade-off: while the policy becomes smoother, the restricted capacity commonly leads to underfitting, degraded control performance, and additional computational overhead during inference.}

\textbf{Offline RL and continuous-time RL.}\quad 
In this study, we primarily focus on offline policy learning, such as offline IL (e.g., behavior cloning, BC) and offline RL, which directly imitate or optimize over a fixed offline dataset without environment interaction. Compared with IL, offline RL is more complex, which optimizes the cumulative return while being regularized towards the behavior data to avoid extrapolation error accumulation~\citep{levine2020offline,kumar2019stabilizing}. Depending on how the learned policy is parameterized, existing offline RL algorithms can be classified into three categories. The first category employs deterministic policies and explicitly constrains the learned actions toward the dataset using distance-based behavior cloning penalties~\citep{fujimoto2021minimalist,li2022data, cheng2024look}. The second and also the more popular category adopts stochastic policies (e.g., Gaussian policies)~\citep{wu2019behavior, kumar2019stabilizing, kumar2020conservative, iql, lyu2022mildly,xuoffline,wang2023offline,xu2022a,xu2022constraints,niu2022when,li2023mind,maoodice,xu2026uni}, which maximizes the likelihood of high-return actions under some distribution-based data regularizations. Recently, the third category of offline RL methods has emerged, which leverages expressive generative models, such as diffusion or flow matching policies, to capture the complex, multimodal action distributions in the datasets~\citep{wang2022diffusion,janner2022planning,hansen2023idql,li2024diffstitch, maodiffusion,zhengsafe,Park2025FlowQ,zheng2025towards,liang2026dipole,zheng2026horizon}. However, all of the above methods are restricted to learning policies in the $0^\text{th}$-order action spaces, which suffer from notable control stability and robustness issues. Incorporating higher-order action supervision in offline RL remains entirely unexplored.
% to perform advantage-weighted imitation learning , aiming to maximize the likelihood of high-return actions through weighted Maximum Likelihood Estimation (MLE). The third category leverages highly expressive generative models, such as diffusion models \citep{wang2022diffusion,janner2022planning,hansen2023idql,li2024diffstitch, maodiffusion,zhengsafe,liu2025skill} and flow matching \citep{Park2025FlowQ}, to capture complex, multimodal action distributions inherent in suboptimal datasets. However, the potential benefits and critical impacts of first-order action variations on policy generalization and control stability remain entirely unexplored within these conventional offline RL frameworks.

Another research avenue closely related to our work is continuous-time RL, which transcends the structural limitations of discrete-time Markov decision processes (MDPs) by modeling state and action evolution continuously over time~\citep{doya2000reinforcement, wang2020reinforcement, munos2006policy}. The Hamilton-Jacobi-Bellman (HJB) equation serves as the core theoretical foundation for this domain \citep{lee2021policy,vamvoudakis2010online, bhasin2013novel, vamvoudakis2014online}. Recent advancements have successfully utilized HJB-based frameworks to demonstrate how continuous-time structural priors can directly guide stable policy evaluation and control \citep{kim2021hamilton, jia2022policy,lien2024enhancing}. Although the HJB-based framework naturally requires $1^\text{st}$-order action $\dot{a}$ information to capture $Q$-value's variation in continuous time, due to the lack of a proper strategy to evaluate $\dot{a}$ from policy $\pi$, existing methods often resort to a crude approximation or replace it with a fixed bounding value. This drastically degrades the value estimation quality under the HJB framework and leads to suboptimal performance.

\section{Extending Standard Policy Classes for Higher-Order Action Supervision}\label{sec:policy_class}
In this paper, unlike prior methods that mitigate action prediction instability by using specialized network architectures or introducing task-specific ad-hoc designs, we propose an elegant, lightweight, and plug-and-play learning scheme to transform a standard policy model (regardless of deterministic, stochastic, or generative flow policies) into a more powerful policy class. Our proposed method enables simultaneous $0^\text{th}$ and $1^\text{st}$-order action supervision for a policy model, maximally utilizing every bit of information in the offline data.

% unlike prior methods that mitigate action fluctuations by designing specialized network architectures or higher-order imitation techniques burdened by the complex computational overhead of Taylor expansions, we propose an extremely minimalist policy class. By fundamentally analyzing error propagation dynamics, we mathematically demonstrate that relying exclusively on zero-order action matching is insufficient to maintain stable control. Motivated by this rigorous analysis, our proposed policy class simultaneously optimizes zero-order and first-order losses during training. Crucially, this yields a highly efficient, plug-and-play module that can be seamlessly integrated into any existing offline reinforcement learning algorithm without introducing additional structural complexity.

% In this section, we propose \emph{higher-order imitation learning}, a simple and novel plug-in for imitation learning that augments standard action matching with supervision on the temporal evolution of expert behavior. Our formulation encourages the learned policy to match the local trajectory geometry induced by the expert. We instantiate this idea across deterministic policies, stochastic Gaussian policies, and flow-based generative policies. This yields a unified and lightweight framework that improves trajectory-level consistency and mitigates compounding errors in closed-loop execution.

\subsection{Notations}
\label{notations}
We consider the sequential decision-making problem modeled as an MDP, defined by the tuple $\mathcal{M} = \{ \mathcal{S}, \mathcal{A}, \mathcal{P}, r, \rho_0, \gamma \}$. $\mathcal{S}$ represents the state space, $\mathcal{A} \subset \mathbb{R}^d$ is the $d$-dimensional continuous action space, $\mathcal{P}(s'|s,a)$ denotes the stochastic transition dynamics, $r: \mathcal{S} \times \mathcal{A} \to \mathbb{R}$ is the reward function, $\rho_0(s)$ specifies the initial state distribution, and $\gamma \in [0, 1)$ is the discount factor. In the offline setting, the agent cannot interact with the environment and is instead provided with a static dataset $\mathcal{D}$ consisting of offline trajectories $\tau = \{s_0, a_0, r_0, s_1, a_1, \dots\}$. In the offline IL setting, we aim to learn an imitative policy $\pi$ to clone the expert behaviors specified in the data trajectories; whereas in offline RL, we aim to learn a policy $\pi$ exclusively from the dataset $\mathcal{D}$ that maximizes the expected cumulative discounted return: $J(\pi) = \mathbb{E}_{\tau \sim \pi} \left[ \sum_{t=0}^{\infty} \gamma^t r(s_t, a_t) \right]$. In this paper, we use $\pi(s)$ to denote the deterministic policy and $\pi(\cdot|s)$ to denote the stochastic policy. Unless told otherwise, we use $a$ to denote the $0^\text{th}$-order action and $\dot{a}=\frac{d a}{d t}$ to denote the $1^\text{st}$-order action. 
In the discrete-time MDP setting, we approximate the time derivatives of the state and actions using finite differences, i.e., $\dot{s}\approx (s'-s)/\Delta t$ and $\dot{a}\approx (a'-a)/\Delta t$, with the time interval set to $\Delta t=1$ by default.

% The objective of offline RL is to learn a parameterized policy $\pi_\theta(a|s)$ exclusively from the dataset $\mathcal{D}$ that maximizes the expected cumulative discounted return: $J(\pi_\theta) = \mathbb{E}_{\tau \sim \pi_\theta} \left[ \sum_{t=0}^{\infty} \gamma^t r(s_t, a_t) \right]$.

% \section{Higher-order imitation learning}
% \label{Higher-order imitation learning}
% \label{methodology}

% with theoretical analysis showing that higher-order supervision limits the evolution of policy mismatch along expert trajectories before it is amplified into closed-loop rollout deviation.
% \subsection{Problem setup}
% \label{setup}
\subsection{Supervising Both Zeroth and First-Order Action for Imitation Learning}
\label{imiation learning}
The core difficulty to achieve simultaneous $0^\text{th}$ and $1^\text{st}$-order action supervision lies in how to properly inject $\dot{a}$ information within the original model structure of the policy, since conventional policy models, regardless of deterministic policy $\pi(s)$ or stochastic policy $\pi(\cdot|s)$ are designed for the straightforward mapping between state $s$ and the $0^\text{th}$-order action $a$. However, if we are allowed to take the time derivative of policy $\pi$, then we will notice the following property (using the simplest deterministic-policy case as an example):
\begin{equation}\label{eq:policy_derivative}
    \dot{a}=\dot{\pi}(s)=\frac{d \pi(s)}{d t}=\frac{d \pi(s)}{d s}\cdot \frac{d s}{d t}=\frac{d \pi(s)}{d s}\dot{s}=\nabla_s\pi(s)\dot{s}.
\end{equation}
The above equation suggests that we can express the time derivative of policy $\pi$ as the product of the gradient of policy $\nabla_s\pi(s)$ and $1^\text{st}$-order state $\dot{s}$. As $\dot{s}$ is directly obtainable from the offline dataset during training, this offers us a new perspective to supervise $1^\text{st}$-order action signal directly on the original policy model. In the following, we describe the detailed higher-order action supervision strategies for deterministic, stochastic, and generative flow-matching policies.

% Behavior cloning is a classical approach in imitation learning.
% However, policies obtained through behavior cloning often exhibit poor performance in closed loop, since small errors in action prediction can drive the system away from the demonstrated state distribution~\citep{ross2011reduction}.
% To mitigate this issue, we incorporate supervision on temporal derivatives, which regularizes the policy not only at the action level but also along the temporal evolution of demonstrated trajectories.
% In the following section, we denote the time derivative of the policy network as $\dot{\pi}(s)=\frac{d \pi(s)}{d s}\dot{s}$. For stochastic policy, we additionally denote $\dot \mu(s) = \frac{d{\mu(s)}}{d{s}}\dot s$ and $\dot \sigma(s) = \frac{d{\sigma(s)}}{d{s}}\dot s$.
% The derivatives of the state and action are approximated using finite differences, i.e., $\dot{s}\approx (s'-s)/\Delta t$ and $\dot{a}\approx (a'-a)/\Delta t$, with the time interval set to $\Delta t=1$ in our implementation.

\paragraph{Deterministic policy.}
% \label{Deterministic imitation learning}
With Eq. (\ref{eq:policy_derivative}), incorporating $1^\text{st}$-order action supervision for deterministic policies under the IL setting is straightforward. We can simply add an extra $1^\text{st}$-order loss term $\mathcal{L}_{\dot{a}}^{\text{det}}$ to the original BC loss:
\begin{equation}\label{eq:det_p_loss}
\min_{\pi}\mathbb{E}_{(s,a,s',a')\sim \mathcal{D}} 
\Big[ 
\left\| \pi(s) - a \right\|^{2}
+
\underbrace{
\left\|
\nabla_s\pi(s)\dot{s}
- \dot{a}
\right\|^{2}
}_{\mathcal{L}_{\dot{a}}^{\text{det}}}
\Big].
\end{equation}
where $\dot{s}=s'-s$ and $\dot{a}=a'-a$ are directly calculated from dataset samples. The second term $\mathcal{L}_{\dot{a}}^{\text{det}}$ regresses $\dot{\pi}(s)$ to the ground truth $\dot{a}$, thereby achieving first-order temporal consistency regularization without the need to make any changes to the original policy network. In practice, we can directly evaluate $\nabla_s\pi(s)\dot{s}$ using Jacobian-vector product (JVP), which is much cheaper to compute than calculating the Jacobian matrix $\nabla_s\pi(s)$, and is natively supported in modern deep learning frameworks like JAX and PyTorch 2.0. A keen reader might notice that we do not introduce any weighting parameter to the two loss terms in Eq. (\ref{eq:det_p_loss}), as we find empirically that keeping the same supervision strength of $0^\text{th}$ and $1^\text{st}$-order actions normally provides the best performance.

% For deterministic policies, conventional imitation learning encourages the policy to directly match expert actions at the visited states~\citep{ross2011reduction}. However, this pointwise supervision provides only a zero-order constraint on the policy and ignores the temporal structure encoded in expert demonstrations. In continuous-control problems such as robotic control, a policy that matches individual actions may still induce inconsistent or nonsmooth action sequences in closed-loop execution, where coherent control variations are crucial~\citep{chi2023diffusion}. To incorporate this missing temporal structure, we augment the standard imitation objective with a first-order term that aligns the temporal variation of the policy with that of the expert action:

% The first term is the standard behavior cloning, whereas the second term is the higher-order plug-in to regularize the local temporal evolution so that the learned policy is consistent with the expert action dynamics along the trajectory.

\paragraph{Stochastic policy.}
Stochastic policies are often modelled as a Gaussian policy $\pi(a|s)=\mathcal{N}(a;\mu(s),\text{diag}(\sigma^2(s)))$ in practice, with $\mu(s)$, $\sigma^2(s)$ being implemented as parameterized neural networks outputting mean and dimension-independent variances. They have been widely used in many offline IL and RL algorithms due to their flexibility in capturing dataset distributions. To enable $1^\text{st}$-order action supervision on a stochastic policy, one na\"ive approach is to 
supervise the mean policy $\mu(s)$ with an additional loss term $\mathcal{L}_{\dot{a}}^{\text{sto-mean}}=\left\|\nabla_s\mu(s)\dot{s}-\dot{a}\right\|$ as in Eq. (\ref{eq:det_p_loss}). We refer to this version as "\textit{mean supervision}". Below, we also provide a more rigorous version consistent with the stochastic process framework.

Specifically, if we model the stochastic component of the Gaussian policy using an Ornstein--Uhlenbeck (OU) \citep{uhlenbeck1930theory} process, $\mathrm{d}X_t = -\lambda X_t \, \mathrm{d}t + \sqrt{2\lambda} \, \mathrm{d}W_t,  X_0 \sim \mathcal{N}(0,I)$, where $\lambda$ controls the temporal correlation of the process, and $W_t$ denotes a d-dimensional standard Wiener process.
% Following the deterministic case, we extend this objective from action matching to first order imitation by modeling the temporal variation of stochastic actions. Given a Gaussian policy $\pi(a|s)=\mathcal{N}(a;\mu(s),\sigma^2(s))$, in the continuous time setting, we represent the stochastic component of the Gaussian policy using an Ornstein--Uhlenbeck (OU) \citep{uhlenbeck1930theory} process, $\mathrm{d}X_t = -\lambda X_t \, \mathrm{d}t + \sqrt{2\lambda} \, \mathrm{d}W_t,  X_0 \sim \mathcal{N}(0,1)$, where $\lambda$ controls the temporal correlation of the process, and $W_t$ denotes a standard Wiener process. 
By the stationarity of the OU process, for any $t$, we have
$X_t\sim \mathcal{N}(0,I)$. The action can thus be written as
$a = \mu(s) + \sigma(s)\odot X_t$, where $\odot$ denotes element-wise
multiplication. By the multi-dimensional Itô formula
\citep{oksendal2013stochastic}, we obtain the stochastic differential equation
(SDE) of the action process,
\begin{equation}
\label{stochastic_da}
\mathrm{d}a
=
\left[
\dot{\mu}(s)
+
\left(
\dot{\sigma}(s)-\lambda\sigma(s)
\right)
\odot
\left(
\left(a-\mu(s)\right)\oslash\sigma(s)
\right)
\right]\mathrm{d}t
+
\sqrt{2\lambda}\,\sigma(s)\odot \mathrm{d}W_t ,
\end{equation}
where $\oslash$ denotes element-wise division. The resulting SDE induces a
first-order policy distribution over the time derivative of actions (see Appendix \ref{app:sto_policy_derivation} for detailed derivation),
\begin{equation} \label{eq:sto_policy_standard}
\dot{\pi}(\dot a|s,a,\dot s)
=
\mathcal{N}\!\left(
\dot a;\,
\dot{\mu}(s)
+
\left(
\dot{\sigma}(s)-\lambda\sigma(s)
\right)
\odot
\left(
\left(a-\mu(s)\right)\oslash\sigma(s)
\right),
\,
2\lambda\operatorname{diag}(\sigma^2(s))
\right),
\end{equation}
where $\dot{\mu}(s)=\nabla_s \mu(s) \cdot \dot{s}$ and $\dot{\sigma}(s)=\nabla_s \sigma(s) \cdot \dot{s}$ can be evaluated in a similar way as in Eq. (\ref{eq:policy_derivative}). The above distribution specifies the likelihood of the time derivative of stochastic actions conditioned on the current state, action, and state derivative. We can thus jointly maximize the likelihood of $0^\text{th}$ and $1^\text{st}$-order action for stochastic policies under the IL setting as follows:
% Combining this derivative-level likelihood with the standard action likelihood yields the stochastic higher-order imitation objective:
\begin{equation}
\label{obj_sto}
\min_{\pi} \mathbb{E}_{(s,a,s',a')\sim \mathcal{D}} \Big[ 
-\log \pi(a|s) 
- \underbrace{\log \dot{\pi}(\dot a|s,a,\dot s)}_{\mathcal{L}_{\dot{a}}^{\text{sto}}} 
\Big].
\end{equation}
Beyond conventional action-level maximum likelihood in the first term, the derivative likelihood $\mathcal{L}_{\dot{a}}^{\text{sto}}$ in Eq.(\ref{obj_sto}) serves as a probabilistic first-order regularizer that aligns the local temporal evolution of sampled actions with the expert trajectory.
We provide the full derivation in Appendix~\ref{app:sto_policy_derivation}.
% where $ \dot{\mu}(s(t)) $ and $\dot{\sigma}(s(t)) $ denote time derivatives. Considering a small time interval $\Delta t $, we obtain,
% \begin{equation} \label{distribution of dot{a}}
% a(t+\Delta t)-a(t) \mid a(t) \sim 
% \mathcal{N}\!\left(
% \Big[ \dot{\mu}(s(t)) 
% + \big( \dot{\sigma}(s(t)) -\lambda \sigma(s(t)) \big)
% \frac{a(t) - \mu(s(t))}{\sigma(s(t)} \Big] \Delta t,\;
% 2{\lambda}\, \sigma^2(s(t))\, \Delta t
% \right).
% \end{equation}
% Based on this, we obtain a first-order Maximum Likelihood Estimation (MLE) formulation,
% \begin{equation}
% \max_{\pi}  \mathbb{E} _{s\sim \mathcal{D} ,a\sim \pi_{\theta }(s),{s}' \sim P(\cdot|s,a),  {a}'\sim \pi_{\theta }(s)} \left [ \mathrm{log} \dot{\pi}   (\dot a |s,a)  \right ] ,
% \end{equation}
% where $\dot{\pi} $ denotes the distribution defined in Eq. \ref{distribution of dot{a}}, while the computations of $\dot{\mu}(s(t))$ and $\dot{\sigma}(s(t)) $ follow the same approach as in Section \ref{Deterministic imitation learning} for deterministic policies. We now present the complete MLE formulation for the Gaussian policy with first-order terms:
% \begin{equation}
% \max_{\pi} \mathbb{E}_{s\sim \mathcal{D},\, a\sim \pi_{\theta}(s),\, s' \sim P(\cdot|s,a),\, a'\sim \pi_{\theta}(s)}
% \Big[
% \log \pi(a|s)
% +
% \underbrace{
% \log \dot{\pi}(\dot{a}|s,a)
% }_{\mathcal{L}_{\mathrm{sto}}}
% \Big].
% \end{equation}
\paragraph{Flow policy.}
Flow matching provides an effective and powerful framework for generative modeling~\citep{lipman2023flow,liu2023flow}, and has gained increasing popularity in recent RL and embodied AI models. 
The central idea is to learn a time-dependent velocity field 
$v: [0,1]\times \mathbb{R}^{d} \rightarrow \mathbb{R}^{d}$ (note this "time" is different from the physical time, which we denote as $k$)
that induces a continuous flow $\psi: [0,1]\times \mathbb{R}^{d} \rightarrow \mathbb{R}^{d}$ between a base distribution and the target data distribution, where $\frac{d}{dk}\psi(k,x)=v\bigl(k,\psi(k,x)\bigr),\psi(0,x)=x$. 

In this paper, we consider the standard linear interpolation path~\citep{lipman2023flow,liu2023flow}, which defines an interpolation process
for \(k \in [0,1]\) as $x^k = (1-k)x^0 + kx^1, v = x^1 - x^0$ between a noise
sample \(x^0\) and a data sample \(x^1\).
For flow policy \citep{Park2025FlowQ}, let \(a^0 \sim \mathcal{N}(0, I)\) denote a noise sample and \(a^1 \sim \mathcal{D}\) denote an action. Given the input state $s$, the generated action is obtained by integrating the learned velocity field along the interpolation path. Note that under the formalism of flow matching, if we differentiate the generated action $a^1 = a^0 + \int_0^1 v(k,s,a^k)\, dk$ with respect to physical time $t$, we have
% To extend flow policy to first order imitation, we differentiate the generated action with respect to physical time:
% \begin{equation}
%     a = \int_0^1 v_k(k,s,x^k)\, dk
% \end{equation}
\begin{equation}
\label{flow_da}
\dot{a}^1
= \frac{d \int_0^1 v(k,s,a^k)\, dk}{dt}
= \int_0^1 \frac{d v(k,s,a^k)}{dt}\, dk
= \int_0^1 \frac{\partial v(k,s,a^k)}{\partial s}\dot{s}\ + \frac{\partial v(k,s,a^k)}{\partial a^k}\dot{a}^k\ dk.
\end{equation}
% We omit $\dot{a}^0$ as the noise sample \(a^0\) is drawn independently of the physical trajectory, hence calculating the time derivative on noise is meaningless. Under the ideal linear flow-matching target, the velocity field satisfies \(v(k,s,a^k)=a^1-a^0\), which is defined by the endpoints and does not explicitly depend on the intermediate point \(a^k\), hence the second term in the final form of Eq.~(\ref{flow_da}) vanishes. Based on that, we find that $\dot{a}$ can be approximated by integrating a first-order vector field $\dot v(k,s,a^k)$ along the interpolation path, analogous to the original action generation process. Meanwhile, the corresponding first-order vector field can be learned by taking the time derivative of the original velocity field, $\dot v(k,s,a^k)=\frac{\partial v(k,s,a^k)}{\partial s}\dot{s}=\dot{a}^1-\dot{a}^0 = \dot{a}^1$, where $\dot a^1 = a'-a^1$. Therefore, the new objective for flow policies with higher-order supervision can be given as:
We omit $\dot{a}^0$ as the noise sample \(a^0\) is drawn independently of the physical trajectory, hence calculating the time derivative on noise is meaningless. Equation~(\ref{flow_da}) shows that first-order action evolution can be expressed by integrating $\dot v(k,s,a^k)$ over flow time, analogous to the original action generation process. This shared structure motivates supervising $\dot v(k,s,a^k)$ in a flow-matching manner by regressing it toward the pointwise target $\dot a^1=a'-a^1$. In practice, we approximate $\dot v(k,s,a^k)$ with $\frac{\partial v(k,s,a^k)}{\partial s}\dot{s}$. We omit the action-dependent term in Eq.~(\ref{flow_da}) because evaluating it increased computational cost without improving performance, as discussed in Appendix~\ref{Additional Ablation Results}. Therefore, the new objective for flow policies with higher-order supervision can be given as:
\begin{equation}\label{eq:flow_policy}
\min_{\pi } \mathbb{E}_{a^0 \sim \mathcal{N}(0, I),\, (s, a^1, s', a') \sim \mathcal{D},\, k\in[0,1]}
\Big[
\left\|v(k,s,a^k) - (a^1 - a^0)\right\|
+
\underbrace{
\left\|
\frac{\partial v(k,s,a^k)}{\partial s}\dot{s}
-
\dot{a}^1
\right\|
}_{\mathcal{L}_{\dot{a}}^{\text{flow}}}
\Big].
\end{equation}
Standard flow matching mainly supervises transformation along the probability path at time $k$, while leaving the physical-time response of the conditional velocity field largely unconstrained. In Eq. (\ref{eq:flow_policy}), the additional term $\mathcal{L}_{\dot{a}}^{\text{flow}}$ completes this missing piece by aligning the state-directional change of the velocity field with the observed action evolution. We provide further discussion in Appendix~\ref{app:flow_policy_derivation}.

% We argue that this missing direction is crucial for control, since the state evolution should not merely change the final sampled action, but coherently deform the entire action-generation flow. The term $\mathcal{L}_{\dot{a}}^{\text{flow}}$ supplies this missing derivative by aligning the state-directional change of the velocity field with the observed action evolution.

\subsection{Theoretical Analysis}
\label{theoretical analysis}
With the additional first-order supervision, we can show in the following theorems that this scheme can lead to a tighter error bound for imitation learning under a set of commonly used assumptions (see Appendix~\ref{assumptions}). Specifically, consider the learned policy $\pi$ and expert policy $\pi^*$ start at the same state $s(0)$  (we denote $s(t)$ as the state in continuous time $t$), and define the policy error as $e(s)=\pi(s)-\pi^*(s)$. We can show that the policy error can be better bounded with stronger $1^\text{st}$-order supervision (see Appendix~\ref{app:theorem2_proof} for the proof):
% overall state deviation of policy $\pi$ amplified after $t$-step rollout is bounded as follows: 
\begin{theorem}
\label{theorem2}
Under Assumption~\ref{ass:domain} in Appendix~\ref{assumptions}, the policy error $e(s^*(t))$ satisfies:
\[
\|e(s^*(t))\|
\le
\|e(s^*(0))\|
+
t \,
\max_{0 \le \tau \le t}
\bigg\|
\nabla_s\pi(s^*(\tau))\dot{s}^*(\tau)
-
\dot{\pi}^*(s^*(\tau))
\bigg\|.
\]
\end{theorem}
Theorem~\ref{theorem2} shows that the growth of policy error $e(s^*)$ on expert trajectory $s^*$ is governed by its initial value and the accumulated $1^\text{st}$-order discrepancy between the learned and expert policies. Note that the standard BC objective does not supervise the second term, hence could lead to substantially increased errors. As Theorem~\ref{theorem2} only provides analysis on the per-state policy error, next, we provide a stronger conclusion to bound the 
state deviation $\delta(t)=s(t)-s^*(t)$ between states from policy rollout $s$ and expert trajectories $s^*$, which captures the closed-loop long-term error accumulation.

% In this sense, $0^\text{th}$-order supervision fixes the policy value on the demonstration manifold, while $1^\text{st}$-order supervision controls how this value drifts along time. Together, the two theorems show that higher-order supervision suppresses imitation error before it is amplified by the closed-loop dynamics.

% The higher-order objectives above are motivated by the temporal nature of closed-loop imitation error. A policy mismatch is not consumed at a single step; it is carried forward by the dynamics, changes future states, and is repeatedly re-queried by the policy. Thus, the key quantity is not only the point-wise action error, but how this error grows along the expert trajectory before being amplified into rollout deviation. To make the temporal propagation explicit, we use $s(t)$ inside the theorem statements, while omitting the time argument in the surrounding discussion.
\begin{theorem}
\label{theorem1}
Under Assumptions~\ref{ass:domain}--\ref{ass:bounded_jacobian} in Appendix~\ref{assumptions},  for $\delta(0)=0$, the state deviation $\delta(t)$ satisfies,
\[
\|\delta(t)\|
\le
\frac{L_a}{\beta}\big(e^{\beta t}-1\big)
\max_{0\le \tau\le t}\|e(s^*(\tau))\|.
\]
% where
% \[
% \beta=L_s+L_aL_{\pi^*}+L_aL_e.
% \]
\end{theorem}
The detailed proof and exact forms of constants $L_a$ and $\beta$ are provided in Appendix~\ref{app:theorem1_proof}. Theorem~\ref{theorem1} decomposes closed-loop policy rollout state deviation into two parts: the policy mismatch error $e(s^*)$ measured on the expert trajectory $s^*$, and the amplification gain $\frac{L_a}{\beta}(e^{\beta t}-1)$ contributed by the environment dynamics, capturing the error amplification due to closed-loop policy-environment interaction. Again, we can see that the long-term state deviation can also be better controlled under higher-order supervision, due to the tighter error bound $\|e(s^*(t))\|$ obtained in Theorem~\ref{theorem2}.
% Hence, compounding error is better viewed as on-manifold policy mismatch being injected into the dynamics over time, rather than as an isolated inaccurate action.

\section{Higher-Order Policy Optimization in Offline RL}
\label{sec:higher_order_policy_optimization}
The extended policy in Section~\ref{sec:policy_class} not only enables higher-order action supervision, but also provides a practical way to evaluate action derivatives $\dot{\pi}_{\theta}(s)$, a quantity that existing offline RL methods never explicitly use. We now show that it has a natural place in critic learning through the Hamilton-Jacobi-Bellman based continuous-time RL  framework, where value learning depends not only on scalar Bellman consistency, but also on how the value changes along the local state-action flow.

% The higher-order policy class in Section~\ref{sec:policy_class} gives us direct access to the policy-induced action flow $\dot{\pi}_{\theta}(s)$, a quantity that standard offline RL never explicitly uses. We now show that this missing quantity has a natural place in critic learning through the continuous-time HJB formulation, where value learning depends not only on scalar Bellman consistency, but also on how the value changes along the local state-action flow. To unify different offline backups, we use $\eta$ to index the backup type: $\eta=\mathrm{pol}$ uses the policy action $\bar{a}_{\eta}(s')=\pi_{\theta}(s')$, while $\eta=\mathrm{in}$ uses the observed next action $\bar{a}_{\eta}(s')=a'$. The corresponding Bellman operator is defined as $\mathcal{T}_{\eta}\hat{Q}(s,a)=r+\gamma\hat{Q}(s',\bar{a}_{\eta}(s'))$.
\subsection{HJB Equation and  Value Learning with First-Order Supervision} 
\label{High-Order Q-learning}
In the HJB-based continuous-time RL framework, the Q-function is defined by~\citep{kim2021hamilton}:
\begin{equation}\label{eq:continuous_Q}
    Q^{\text{HJB}}(s,a):=\max_{a\in \mathcal{A}}\int_0^\infty e^{-\gamma t}r(s(t),a(t))dt.
\end{equation}
The dynamic programming equation of the above Q-function is expressed as an \textit{HJB equation}, which is derived by using the Taylor expansion: 
\begin{equation} \label{HJB}
\gamma Q^{\mathrm{HJB} }(s,a)=\max_{a\in\mathcal{A}}\left [  r(s,a) + \nabla_{s} Q^{\mathrm{HJB} }(s,a) \dot{s} + \nabla_{a} Q^{\mathrm{HJB} }(s,a) \dot{a}   \right ] 
\end{equation}
% We start from the HJB characterization of the action-value function in continuous-time control~\citep{kim2021hamilton}:
% where $\gamma'$ denotes the continuous-time discount rate. 
Unlike the typical discrete-time Bellman equation, which compresses value learning into a scalar one-step regression, Eq.~(\ref{HJB}) reveals what is missing in existing offline RL methods. The value function may match $0^\text{th}$-order Bellman targets, yet its gradient is never required to align with the variation direction of states and actions in time. This makes the learned value landscape numerically plausible but structurally under-specified, especially for later policy improvement.
% of the underlying decision-making problem. 
However, with the extended policy class introduced in Section~\ref{sec:policy_class}, we can easily obtain the time derivative of actions through the learned policy, i.e., $\dot{a}=\dot{\pi}_{\theta}(s)=\nabla_s\pi_{\theta}(s)\cdot \dot{s}$. This allows us to jointly optimize the value function with policy's higher-order information during Bellman update, which leads to the following objective:
% This makes the learned value landscape numerically plausible but structurally under-specified, especially for later policy improvement. The HJB form turns this weakness into a useful design principle. 
% It indicates that higher-order action information should enter critic learning through the action-directional derivative $\nabla_a Q(s,a)\dot{a}$. In the discrete-time offline setting, we obtain $\dot{s}=s'-s$ directly from the dataset and use the higher-order policy class from Section~\ref{sec:policy_class} to provide the policy-induced action flow $\dot{a}=\dot{\pi}_{\theta}(s)$. 
\begin{equation} 
\label{eq:high_order_q}
\begin{aligned}
Q = \arg\min_{Q} \mathbb{E}_{(s,a,r,s',a')\sim \mathcal{D}} \bigg[
&\left( \mathcal{T}_{\pi}\hat{Q}(s,a) - Q(s,a) \right)^2 \\
&{\color{blue}- 
\lambda_Q \left(
\frac{\nabla_s Q(s,a) \dot s}
{\|\nabla_s Q(s,a)\| \|\dot s\|} 
+ 
\frac{\nabla_a Q(s,a) \dot{\pi}_{\theta}(s)}
{\|\nabla_a Q(s,a)\| \|\dot{\pi}_{\theta}(s)\|}\right)}
\bigg].
\end{aligned}
\end{equation}
As $\dot{s}=s'-s$ can be directly calculated from the dataset during offline learning, we can actually seamlessly transform the original value learning scheme of an existing offline RL algorithm into a higher-order supervision version by adding the above {\color{blue}blue terms}, with $\mathcal{T}_{\pi}$ denoting the specific Bellman operator used in the offline RL algorithm. Moreover, to improve numerical stability and to prevent overly large gradient magnitudes of the first-order terms, we normalize the two HJB directional terms in Eq. (\ref{eq:high_order_q}), similar to existing HJB-based methods~\cite{kim2021hamilton,lien2024enhancing}.

\vspace{-8pt}
\subsection{Higher-Order Policy Learning}
\vspace{-2pt}
\label{Offline RL with plug-in high-order supervision}
With the $Q$-function learned with $1^\text{st}$-order supervision, we can also transform the policy optimization procedure of an existing offline RL algorithm through a plug-and-play modification. This can be achieved by simply adding an additional $1^\text{st}$-order imitation loss $\mathcal{L}_{\dot a}$ as defined in Section~\ref{sec:policy_class},
\begin{equation}
    \pi =\underset{\pi}{\arg\max}\;\mathbb{E}_{(s,a,s',a')\sim\mathcal{D},\hat{a}\sim\pi}
    Q(s,\hat{a})
    -
    \mathcal{R}_{\pi}(s)
    {\color{blue}-\lambda_{\dot a} \mathcal{L}_{\dot a}}.
\end{equation}
where we use $Q(s,\hat{a})-\mathcal{R}_{\pi}(s)$ to represent the general policy optimization objective of existing offline RL algorithms, which typically optimize a policy by maximizing the Q-function with some behavior regularization term $\mathcal{R}_{\pi}(s)$. $\mathcal{R}_{\pi}(s)$ can either be an explicit policy constraint, such as the BC term $\mathcal{R}_{\pi}(s)=\mathbb{E}_{(s,a)\sim\mathcal{D}}\|\pi(s)-a\|^2$ in TD3+BC~\citep{fujimoto2021minimalist}, or an implicit behavior regularization scheme as adopted in many algorithms~\cite{iql,xuoffline,wang2024offline,zhengsafe,xu2026uni,liang2026dipole}, which results in an advantage-weighted regression type of learning objective. Depending on the actual policy class used in the original offline RL algorithm, one can flexibly choose $\mathcal{L}_{\dot a}$ to be either $\mathcal{L}_{\dot a}^{\mathrm{det}}$, $\mathcal{L}_{\dot a}^{\mathrm{sto}}$, or $\mathcal{L}_{\dot a}^{\mathrm{flow}}$ in Section~\ref{sec:policy_class}. We introduce a hyperparameter $\lambda_{\dot a}$ to balance the regularization strength of the original policy optimization objective and the first-order supervision term. Both the first-order value supervision term and the $\mathcal{L}_{\dot a}$ term in policy optimization introduce minimal modifications to the original offline RL algorithm, but as we show in the experiments, they can lead to substantial performance gains, as the policy is taught not only what to do, but how to evolve reasonably over time.

\vspace{-5pt}
\section{Experiments}
\vspace{-4pt}
\label{sec:experiment}
In this section, we empirically evaluate our proposed policy class. We first integrate it into a range of representative offline RL algorithms and benchmark the resulting variants against their original counterparts. We then conduct ablation studies to isolate the contributions of the first-order policy constraint and the first-order value consistency regularizer.
% In this section, we systematically evaluate the empirical performance of our proposed policy class. We begin by demonstrating the performance improvements achieved when integrating our method into existing offline RL algorithms across a diverse suite of benchmarks. Subsequently, we conduct comprehensive ablation studies to isolate the contributions of the first-order policy constraint during policy optimization and the first-order value consistency for value learning.
\vspace{-3pt}
\subsection{Experimental Setup}
\vspace{-2pt}
\textbf{Benchmarks.} To comprehensively evaluate our proposed framework, we conduct experiments on two widely adopted offline RL benchmarks. We first evaluate on OGBench~\citep{ogbench_park2025}, a recent benchmark featuring high-dimensional state spaces, multimodal behaviors, and sparse rewards. We adopt its reward-based "-singletask" variants so that standard reward-maximizing offline RL algorithms can be directly applied. We further conduct experiments on D4RL~\citep{fu2020d4rl} under a more challenging low-data setting with only 10k transitions per task, which better exposes performance gaps and OOD generalization since recent offline RL methods have largely saturated the D4RL benchmark under the full-dataset setting. The benchmark details and task descriptions are presented in Appendix~\ref{app:benchmark_details}.

% To comprehensively evaluate our proposed framework, we conduct experiments on two widely adopted offline RL benchmarks: D4RL~\citep{fu2020d4rl} and OGBench~\citep{ogbench_park2025}. While D4RL has been instrumental in the field, current offline RL algorithms have largely saturated its performance limits. To rigorously assess our method's capabilities, we place a primary focus on the recently introduced OGBench~\citep{ogbench_park2025} and a more challenging low-data setting in D4RL as proposed in recent works~\citep{cheng2024look}. OGBench features high-dimensional state-action spaces, complex multimodal data distributions, and significantly sparser rewards compared to standard D4RL tasks. We utilize the reward-based single-task variants ("-singletask") of OGBench to ensure compatibility with standard reward-maximizing offline RL algorithms.

\textbf{Comparison methods.}
We integrate our proposed policy class into representative offline RL algorithms spanning three policy families (deterministic, stochastic, and flow-matching), and compare each base algorithm with its first-order-supervised variant.
For deterministic policies, we adopt TD3+BC~\citep{fujimoto2021minimalist} as the base algorithm.
For stochastic policies, we instantiate the two strategies introduced in Section~\ref{imiation learning}: a simplified \textit{mean supervision} variant, which applies first-order supervision only to the Gaussian mean and is built upon ReBRAC~\citep{tarasov2023revisiting} (denoted ``mean sup.'' in Table~\ref{tab:OGbench results}), and a \textit{standard} variant based on the likelihood objective in Eq.~\ref{obj_sto} and built upon IQL~\citep{iql} (denoted ``standard'').
For flow-matching policies, we adopt IFQL, an algorithmic variant of FQL~\citep{Park2025FlowQ} that enables our first-order supervision to be directly applied to the flow policy without any additional policy-distillation step.
Experimental details are provided in Appendix~\ref{app:imp_details}.

\vspace{-3pt}
\subsection{Evaluation Results}
\vspace{-3pt}
\begin{table}[t]
    \caption{\small Average normalized scores on 10 OGBench task categories, each aggregating 5 sub-tasks over 5 random seeds. Values after $\pm$ denote standard deviations. The ${\color{mygreen}\uparrow}$ marks indicate improvements greater than $1.0$, whereas ${\color{gray}\uparrow}$ denotes minor improvements. Full per-task results are presented in Table~\ref{tab:OGbench full results}.}
    \label{tab:OGbench results}
    \begin{center}
        \resizebox{\textwidth}{!}{
            \begin{tabular}{lcccccccc}
                \toprule
                & \multicolumn{2}{c}{Det. Policy} 
                & \multicolumn{2}{c}{Sto. Policy (mean sup.)}
                & \multicolumn{2}{c}{Sto. Policy (standard)}
                & \multicolumn{2}{c}{Flow Policy} \\
                \cmidrule(lr){2-3} \cmidrule(lr){4-5} \cmidrule(lr){6-7} \cmidrule(lr){8-9}
                
                Task Category (Avg. of 5 tasks)
                & TD3+BC & TD3+BC+ours 
                & ReBRAC & ReBRAC+ours & IQL 
                & IQL+ours & IFQL & IFQL+ours \\
                \midrule
                antmaze-large-navigate  &$83{\scriptstyle \pm 4}$ &$98{\scriptstyle \pm 3}({\color{mygreen}\uparrow\! 15})$ &$81{\scriptstyle \pm 5}$ &$98{\scriptstyle \pm 2}({\color{mygreen}\uparrow\! 17})$  &$53{\scriptstyle \pm 3}$  &$68{\scriptstyle \pm 5}({\color{mygreen}\uparrow\! 15})$  &$28{\scriptstyle \pm 13}$  & $36{\scriptstyle \pm 14}({\color{mygreen}\uparrow\! 8})$ \\
                antmaze-giant-navigate &$42{\scriptstyle \pm 8}$ & $65{\scriptstyle \pm 10}({\color{mygreen}\uparrow\! 23})$&$26{\scriptstyle \pm 8}$   &$61{\scriptstyle \pm 12}({\color{mygreen}\uparrow\! 35})$  &$4{\scriptstyle \pm 1}$  &$9{\scriptstyle \pm 1}({\color{mygreen}\uparrow\! 5})$  & $3{\scriptstyle \pm 2}$ & $4{\scriptstyle \pm 3}({\color{mygreen}\uparrow\! 1})$ \\
                humanoidmaze-medium-navigate &$34{\scriptstyle \pm 8}$ &$42{\scriptstyle \pm 5}({\color{mygreen}\uparrow\! 8})$ &$22{\scriptstyle \pm 8}$   &$55{\scriptstyle \pm 10}({\color{mygreen}\uparrow\! 33})$  & $33{\scriptstyle \pm 2}$ &$36{\scriptstyle \pm 3}({\color{mygreen}\uparrow\! 3})$  & $60{\scriptstyle \pm 16}$ & $64{\scriptstyle \pm 13}({\color{mygreen}\uparrow\! 4})$ \\
                humanoidmaze-large-navigate &$2{\scriptstyle \pm 1}$ &$4{\scriptstyle \pm 1}({\color{mygreen}\uparrow\! 2})$ &$2{\scriptstyle \pm 1}$  &$4{\scriptstyle \pm 2}({\color{mygreen}\uparrow\! 2})$  &$2{\scriptstyle \pm 1}$  &$3{\scriptstyle \pm 1}({\color{mygreen}\uparrow\! 1})$  & $11{\scriptstyle \pm 2}$ & $12{\scriptstyle \pm 3}({\color{mygreen}\uparrow\! 1})$ \\
                antsoccer-arena-navigate &$13{\scriptstyle \pm 3}$ &$16{\scriptstyle \pm 3}({\color{mygreen}\uparrow\! 3})$ &$0{\scriptstyle \pm 0}$  &$16{\scriptstyle \pm 4}({\color{mygreen}\uparrow\! 16})$  &$8{\scriptstyle \pm 2}$  &$23{\scriptstyle \pm 3}({\color{mygreen}\uparrow\! 15})$  & $33{\scriptstyle \pm 6}$ & $40{\scriptstyle \pm 13}({\color{mygreen}\uparrow\! 7})$ \\
                cube-single-play &$88{\scriptstyle \pm 5}$ &$91{\scriptstyle \pm 4}({\color{mygreen}\uparrow\! 3})$  &$91{\scriptstyle \pm 2}$  &$93{\scriptstyle \pm 2}({\color{mygreen}\uparrow\! 2})$  &$83{\scriptstyle \pm 3}$  &$92{\scriptstyle \pm 3}({\color{mygreen}\uparrow\! 9})$  & $79{\scriptstyle \pm 2}$ & $95{\scriptstyle \pm 7}({\color{mygreen}\uparrow\! 16})$  \\
                cube-double-play &$7{\scriptstyle \pm 1}$ &$12{\scriptstyle \pm 1}({\color{mygreen}\uparrow\! 5})$ &$12{\scriptstyle \pm 1}$  &$24{\scriptstyle \pm 3}({\color{mygreen}\uparrow\! 12})$  &$7{\scriptstyle \pm 1}$  &$8{\scriptstyle \pm 2}({\color{mygreen}\uparrow\! 1})$  & $14{\scriptstyle \pm 3}$ & $42{\scriptstyle \pm 15}({\color{mygreen}\uparrow\! 28})$ \\
                scene-play &$44{\scriptstyle \pm 3}$ &$47{\scriptstyle \pm 2}({\color{mygreen}\uparrow\! 3})$ &$41{\scriptstyle \pm 3}$  &$47{\scriptstyle \pm 4}({\color{mygreen}\uparrow\! 6})$  &$28{\scriptstyle \pm 1}$  &$33{\scriptstyle \pm 3}({\color{mygreen}\uparrow\! 5})$  & $30{\scriptstyle \pm 3}$ & $56{\scriptstyle \pm 7}({\color{mygreen}\uparrow\! 26})$ \\
                puzzle-3x3-play &$27{\scriptstyle \pm 3}$ &$29{\scriptstyle \pm 1}({\color{mygreen}\uparrow\! 2})$ &$21{\scriptstyle \pm 1}$  &$24{\scriptstyle \pm 2}({\color{mygreen}\uparrow\! 3})$  &$9{\scriptstyle \pm 1}$  &$12{\scriptstyle \pm 4}({\color{mygreen}\uparrow\! 3})$  & $19{\scriptstyle \pm 1}$ & $43{\scriptstyle \pm 15}({\color{mygreen}\uparrow\! 24})$ \\
                puzzle-4x4-play &$6{\scriptstyle \pm 1}$ &$9{\scriptstyle \pm 1}({\color{mygreen}\uparrow\! 3})$ &$14{\scriptstyle \pm 1}$  &$17{\scriptstyle \pm 3}({\color{mygreen}\uparrow\! 3})$  &$7{\scriptstyle \pm 1}$  &$13{\scriptstyle \pm 2}({\color{mygreen}\uparrow\! 6})$   & $25{\scriptstyle \pm 5}$ &$40{\scriptstyle \pm 12}({\color{mygreen}\uparrow\!15})$  \\
                \midrule
                Average Score &35 &41$({\color{mygreen}\uparrow\! 6})$ &31  &44$({\color{mygreen}\uparrow\! 13})$  &23  &30$({\color{mygreen}\uparrow\! 7})$  &30  &43$({\color{mygreen}\uparrow\! 13})$  \\
                \bottomrule
            \end{tabular}
        }
    \end{center}
   \vspace{-8pt} % 控制两个表格之间的垂直间距
    \caption{\small Average normalized scores on D4RL MuJoCo locomotion tasks with reduced-size (10k) datasets. Standard deviations are reported after the $\pm$ sign. The ${\color{mygreen}\uparrow}$ symbol indicates a performance increase greater than $1.0$, whereas ${\color{gray}\uparrow}$ denotes minor improvements. All evaluation results are averaged over 5 random seeds.}
    \label{tab:d4rl_res_10k}
    \begin{center}
        \resizebox{\textwidth}{!}{
            \begin{tabular}{lcccccccc}
            \toprule
            & \multicolumn{2}{c}{Det. Policy} & \multicolumn{2}{c}{Sto. Policy (mean sup.)} & \multicolumn{2}{c}{Sto. Policy (standard)} & \multicolumn{2}{c}{Flow Policy} \\
            \cmidrule(lr){2-3} \cmidrule(lr){4-5} \cmidrule(lr){6-7} \cmidrule(lr){8-9}
            Task Category  & TD3+BC        & TD3+BC+ours          & ReBRAC            & ReBRAC+ours        & IQL            & IQL+ours                  & IFQL & IFQL+ours            \\ 
            \midrule
            Hopper-m & $40.1{\scriptstyle \pm 18.6}$ & $44.1{\scriptstyle \pm 12.7}({\color{mygreen}\uparrow\! 4.0})$ & $9.8{\scriptstyle \pm 3.1}$ & $32.4{\scriptstyle \pm 11.0}({\color{mygreen}\uparrow\! 22.6})$ & $47.7{\scriptstyle \pm 1.0}$ & $48.5{\scriptstyle \pm 2.5}({\color{gray}\uparrow\! 0.8})$ & $50.8{\scriptstyle \pm 2.3}$ & $51.3{\scriptstyle \pm 2.9}({\color{gray}\uparrow\! 0.5})$ \\
            Hopper-mr & $7.3{\scriptstyle \pm 6.1}$ & $22.3{\scriptstyle \pm 3.3}({\color{mygreen}\uparrow\! 15.0})$ & $19.0{\scriptstyle \pm 1.0}$ & $24.9{\scriptstyle \pm 2.8}({\color{mygreen}\uparrow\! 5.9})$ & $19.3{\scriptstyle \pm 3.0}$ & $19.9{\scriptstyle \pm 2.0}({\color{gray}\uparrow\! 0.6})$ & $23.1{\scriptstyle \pm 1.4}$ & $28.2{\scriptstyle \pm 0.4}({\color{mygreen}\uparrow\! 5.1})$ \\
            Hopper-me & $17.8{\scriptstyle \pm 7.9}$ & $42.3{\scriptstyle \pm 10.8}({\color{mygreen}\uparrow\! 24.5})$ & $36.2{\scriptstyle \pm 3.1}$ & $43.0{\scriptstyle \pm 3.3}({\color{mygreen}\uparrow\! 6.8})$ & $34.3{\scriptstyle \pm 8.7}$ & $42.3{\scriptstyle \pm 3.3}({\color{mygreen}\uparrow\! 8.0})$ & $45.7{\scriptstyle \pm 1.9}$ & $47.9{\scriptstyle \pm 1.7}({\color{mygreen}\uparrow\! 2.2})$ \\
            Hopper-e & $23.2{\scriptstyle \pm 18.2}$ & $93.5{\scriptstyle \pm 13.6}({\color{mygreen}\uparrow\! 70.3})$ & $59.2{\scriptstyle \pm 3.5}$ & $97.2{\scriptstyle \pm 5.9}({\color{mygreen}\uparrow\! 38.0})$ & $55.7{\scriptstyle \pm 22.5}$ & $99.5{\scriptstyle \pm 3.2}({\color{mygreen}\uparrow\! 43.8})$ & $78.5{\scriptstyle \pm 13.6}$ & $96.9{\scriptstyle \pm 4.9}({\color{mygreen}\uparrow\! 18.4})$ \\
            \midrule
            Halfcheetah-m & $16.4{\scriptstyle \pm 10.2}$ & $31.5{\scriptstyle \pm 2.4}({\color{mygreen}\uparrow\! 15.1})$ & $37.1{\scriptstyle \pm 1.3}$ & $38.9{\scriptstyle \pm 0.5}({\color{mygreen}\uparrow\! 1.8})$ & $38.2{\scriptstyle \pm 1.4}$ & $39.1{\scriptstyle \pm 0.1}({\color{gray}\uparrow\! 0.9})$ & $38.9{\scriptstyle \pm 1.4}$ & $39.5{\scriptstyle \pm 1.5}({\color{gray}\uparrow\! 0.6})$ \\
            Halfcheetah-mr & $17.9{\scriptstyle \pm 9.5}$ & $21.2{\scriptstyle \pm 7.1}({\color{mygreen}\uparrow\! 3.3})$ & $13.0{\scriptstyle \pm 5.7}$ & $15.1{\scriptstyle \pm 4.1}({\color{mygreen}\uparrow\! 2.1})$ & $28.2{\scriptstyle \pm 1.3}$ & $29.5{\scriptstyle \pm 0.6}({\color{mygreen}\uparrow\! 1.3})$ & $29.6{\scriptstyle \pm 1.0}$ & $31.1{\scriptstyle \pm 1.2}({\color{mygreen}\uparrow\! 1.5})$ \\
            Halfcheetah-me & $15.4{\scriptstyle \pm 10.7}$ & $35.6{\scriptstyle \pm 2.4}({\color{mygreen}\uparrow\! 20.2})$ & $12.2{\scriptstyle \pm 6.7}$ & $19.7{\scriptstyle \pm 1.7}({\color{mygreen}\uparrow\! 7.5})$ & $34.4{\scriptstyle \pm 2.2}$ & $36.8{\scriptstyle \pm 1.5}({\color{mygreen}\uparrow\! 2.4})$ & $40.2{\scriptstyle \pm 0.7}$ & $40.4{\scriptstyle \pm 1.3}({\color{gray}\uparrow\! 0.2})$ \\
            Halfcheetah-e & $1.7{\scriptstyle \pm 3.3}$ & $3.9{\scriptstyle \pm 2.1}({\color{mygreen}\uparrow\! 2.2})$ & $2.3{\scriptstyle \pm 0.9}$ & $4.0{\scriptstyle \pm 2.3}({\color{mygreen}\uparrow\! 1.7})$ & $1.2{\scriptstyle \pm 1.0}$ & $6.7{\scriptstyle \pm 2.6}({\color{mygreen}\uparrow\! 5.5})$ & $18.8{\scriptstyle \pm 7.8}$ & $35.9{\scriptstyle \pm 6.2}({\color{mygreen}\uparrow\! 17.1})$ \\
            \midrule
            Walker2d-m & $7.4{\scriptstyle \pm 13.1}$ & $41.1{\scriptstyle \pm 10.8}({\color{mygreen}\uparrow\! 33.7})$ & $21.2{\scriptstyle \pm 4.8}$ & $56.9{\scriptstyle \pm 4.3}({\color{mygreen}\uparrow\! 35.7})$ & $34.6{\scriptstyle \pm 9.5}$ & $45.0{\scriptstyle \pm 4.8}({\color{mygreen}\uparrow\! 10.4})$ & $50.6{\scriptstyle \pm 6.8}$ & $56.7{\scriptstyle \pm 4.7}({\color{mygreen}\uparrow\! 6.1})$ \\
            Walker2d-mr & $5.7{\scriptstyle \pm 5.8}$ & $23.2{\scriptstyle \pm 5.7}({\color{mygreen}\uparrow\! 17.5})$ & $10.0{\scriptstyle \pm 2.2}$ & $22.8{\scriptstyle \pm 5.5}({\color{mygreen}\uparrow\! 12.8})$ & $10.7{\scriptstyle \pm 11.9}$ & $17.6{\scriptstyle \pm 4.0}({\color{mygreen}\uparrow\! 6.9})$ & $17.5{\scriptstyle \pm 5.4}$ & $29.0{\scriptstyle \pm 2.8}({\color{mygreen}\uparrow\! 11.5})$ \\
            Walker2d-me & $7.9{\scriptstyle \pm 9.1}$ & $37.8{\scriptstyle \pm 12.3}({\color{mygreen}\uparrow\! 29.9})$ & $17.6{\scriptstyle \pm 8.6}$ & $58.9{\scriptstyle \pm 11.8}({\color{mygreen}\uparrow\! 41.3})$ & $29.1{\scriptstyle \pm 16.9}$ & $41.2{\scriptstyle \pm 9.2}({\color{mygreen}\uparrow\! 12.1})$ & $38.8{\scriptstyle \pm 3.7}$ & $57.6{\scriptstyle \pm 10.4}({\color{mygreen}\uparrow\! 18.8})$ \\
            Walker2d-e & $2.7{\scriptstyle \pm 3.1}$ & $82.3{\scriptstyle \pm 10.8}({\color{mygreen}\uparrow\! 79.5})$ & $23.6{\scriptstyle \pm 9.4}$ & $83.9{\scriptstyle \pm 6.8}({\color{mygreen}\uparrow\! 60.3})$ & $43.1{\scriptstyle \pm 6.0}$ & $84.8{\scriptstyle \pm 13.9}({\color{mygreen}\uparrow\! 41.7})$ & $95.2{\scriptstyle \pm 11.8}$ & $104.9{\scriptstyle \pm 3.7}({\color{mygreen}\uparrow\! 9.7})$ \\
            \midrule
            Average Score &$13.6$    &$39.9({\color{mygreen}\uparrow\! 26.3})$   &$21.8$   &$41.5({\color{mygreen}\uparrow\! 19.7})$   &$31.4$   &$42.6({\color{mygreen}\uparrow\! 10.2})$   &$44.0$   &$51.6({\color{mygreen}\uparrow\! 7.6})$ \\
            \bottomrule
            \end{tabular}
        }
    \end{center}
\end{table}

\textbf{Evaluation results on OGBench tasks.}
Table~\ref{tab:OGbench results} reports the average normalized scores on 50 diverse OGBench tasks, with complete per-task results provided in Appendix~\ref{Additional Evaluation Results}. Our first-order action supervision yields consistent performance gains across all three policy paradigms. On control-sensitive navigation tasks, these gains underline the importance of temporally consistent actions, since small action fluctuations can accumulate over time and easily lead to collisions or failed navigation. For instance, applying our scheme to ReBRAC raises the score from 26 to 61 on \texttt{antmaze-giant-navigate} and from 22 to 55 on \texttt{humanoidmaze-medium-navigate}, both revealing over 100\% relative improvement.

The performance gains are even more pronounced beyond navigation. On object-centric play tasks, applying our proposed scheme to IFQL improves \texttt{cube-double-play} from 14 to 42 (a 200\% relative improvement) and \texttt{scene-play} from 30 to 56 (an 87\% relative improvement). In particular, on several challenging sub-tasks where the IFQL baseline completely fails with a score of exactly 0, equipping it with our scheme lifts the score to clearly non-trivial levels, e.g., from 0 to 88 on \texttt{scene-play-singletask-task2-v0}, expanding the set of tasks on which offline policy learning is feasible at all (see Table~\ref{tab:OGbench full results} in Appendix~\ref{Additional Evaluation Results} for full per-task results).

\textbf{Improved OOD generalization under small-dataset settings.}
Our first-order action supervision is particularly beneficial in small-dataset regimes, where standard offline RL baselines often struggle to recover the expert's robust control behavior from sparse state-action samples, even though such samples still encode rich local variations along expert trajectories. Table~\ref{tab:d4rl_res_10k} reports the D4RL locomotion results on 10k datasets. By supervising the policy's first-order action derivative, our scheme enables the local dynamics and control structure of the expert behavior to be effectively recovered from an extremely limited dataset, leading to substantial performance gains. For instance, on \texttt{hopper-expert-v2}, IQL improves from $55.7$ to $99.5$, nearly recovering its full-dataset performance. More notably, on \texttt{walker2d-expert-v2}, TD3+BC collapses to a score of only $2.7$, indicating severe extrapolation failure under such a tiny-data regime; equipping it with our scheme recovers the score to $82.3$, an absolute gain of $79.5$ points. In addition to the score improvements, our scheme also yields noticeably smaller standard deviations across most tasks, indicating more stable and reproducible policy learning under small-dataset settings. Additional limited-data imitation learning results, as well as full-data D4RL results, are provided in Appendix~\ref{Additional Evaluation Results}.

\textbf{Benefits of utilizing first-order supervision.}
Figure~\ref{fig:compare_da} provides a mechanistic view of how our first-order action supervision improves performance. With only 10k transitions, deterministic BC produces a dispersed distribution of $||\dot{a}||$, reflecting erratic action changes during execution. In contrast, our scheme yields a compact distribution concentrated around small $||\dot{a}||$, indicating smoother and temporally consistent action evolution. This pattern is even more pronounced for TD3+BC in Figure~\ref{fig:compare_td3bc}, where its plain MSE penalty leaves the policy vulnerable to action jittering, whereas our scheme directly aligns action changes with the local expert dynamics. Overall, by constraining how actions evolve along trajectories, first-order supervision yields noticeably smoother and more robust policies. We conduct additional analysis results in Appendix~\ref{Additional Evaluation Results}.

\begin{figure}[t!]
    \centering
    \includegraphics[width=\linewidth]{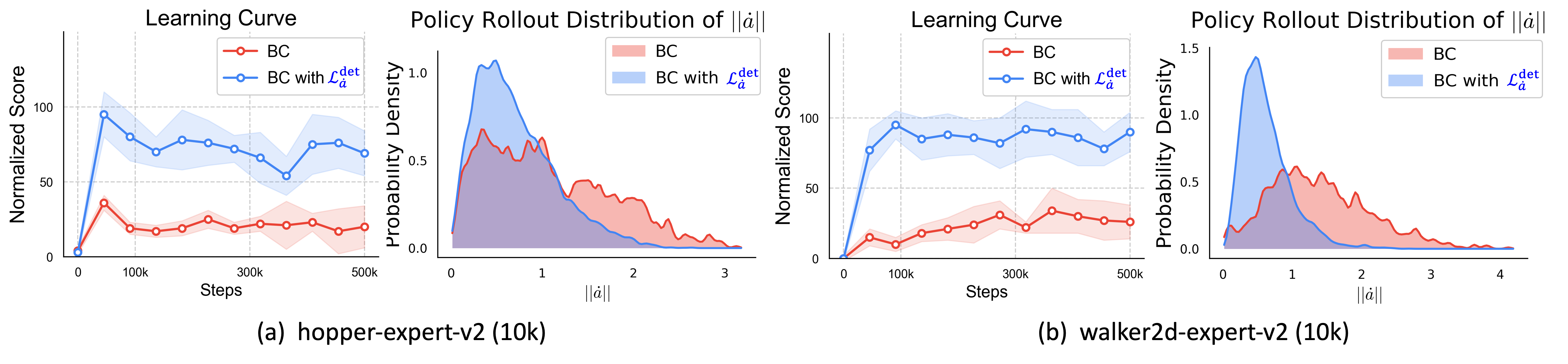}
    \vspace{-18pt}
    \caption{\small Comparison of vanilla deterministic BC and our $1^{\mathrm{st}}$-order action supervision on small-sample locomotion tasks. Scores are averaged over 5 random seeds. The policy rollout distributions of action variation $||\dot{a}||=||a'-a||$ ($a$ is joint torque) are aggregated from 10 trajectories.}
    \label{fig:compare_da}
    \vspace{-6pt}
\end{figure}
\begin{figure}[t!]
    \centering
    \includegraphics[width=\linewidth]{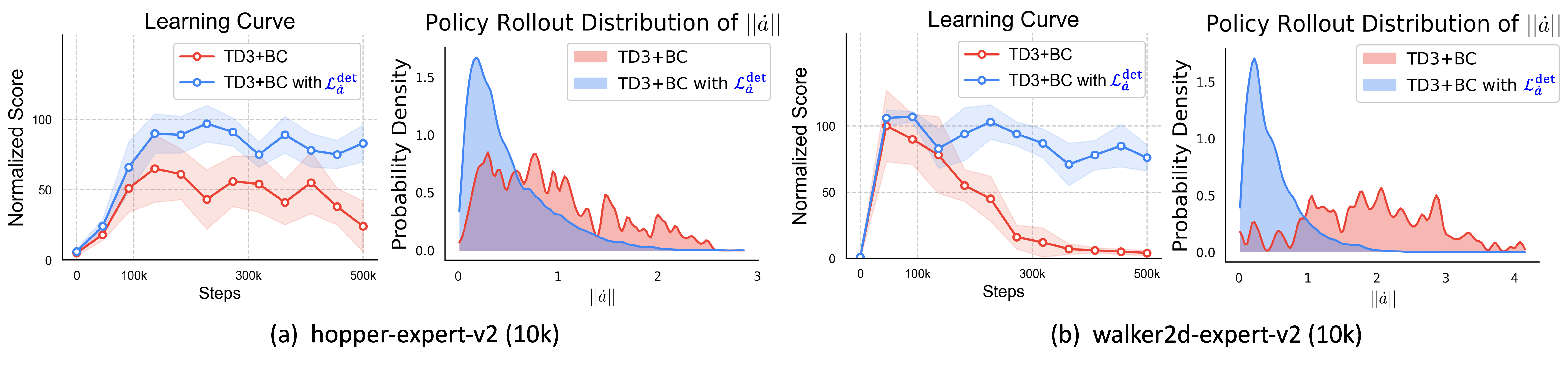}
    \vspace{-18pt}
    \caption{\small Comparison of TD3+BC and our $1^{\mathrm{st}}$-order action supervision on small-sample locomotion tasks. Scores are averaged over 5 random seeds. The policy rollout distributions of action variation $||\dot{a}||=||a'-a||$ ($a$ is joint torque) are aggregated from 10 trajectories.}
    \label{fig:compare_td3bc}
    % \vspace{-6pt}
\end{figure}

% The fundamental mechanism behind this enhanced robustness is explicitly revealed by the policy rollout distributions of the first-order action $||\dot{a}||$ shown in Figure~\ref{fig:compare_da}. Baseline algorithms exhibit a broad, long-tailed distribution of $||\dot{a}||$, which physically corresponds to noisy, temporally inconsistent, and bang-bang-like torque switching. By jointly supervising both zeroth- and first-order actions, our method heavily penalizes these high-frequency action fluctuations. This structural regularization naturally yields a much more compact, low-$||\dot{a}||$ distribution. Consequently, the learned policy is explicitly constrained to output smooth and stable transitions, mitigating closed-loop error amplification and ensuring reliable control even when the offline data density is severely impoverished.
% \vspace{-6pt}
\subsection{Ablation experiments}

\begin{table}[t]
    \centering
    \newcommand{\abup}[1]{({\color{mygreen}\uparrow\! #1})}
    \newcommand{\abminor}[1]{({\color{gray}\uparrow\! #1})}
    \newcommand{\abdown}[1]{({\color{red}\downarrow\! #1})}

    \caption{\small Comparison with action-smoothness approaches and explicit
    first-order action prediction on D4RL MuJoCo locomotion tasks with
    reduced-size (10k) datasets. All methods use TD3+BC as the shared backbone.
    Arrows show score differences relative to TD3+BC on the same task:
    ${\color{mygreen}\uparrow}$ indicates a gain of at least $1.0$,
    ${\color{gray}\uparrow}$ a smaller positive gain, and
    ${\color{red}\downarrow}$ a decrease. No arrow indicates no change.}
    \label{tab:ablation_baselines}

    \resizebox{\textwidth}{!}{%
        \begin{tabular}{lc|ccc|c|c}
            \toprule
            Task & TD3+BC & +CAPS & +L2C2 & +LipsNet
            & +Extra $\dot{a}$ head & TD3+BC+Ours \\
            \midrule
            Hopper-m
            & $40.1{\scriptstyle \pm 18.6}$
            & $42.6{\scriptstyle \pm 5.1}\abup{2.5}$
            & $42.7{\scriptstyle \pm 7.5}\abup{2.6}$
            & $43.2{\scriptstyle \pm 3.6}\abup{3.1}$
            & $41.5{\scriptstyle \pm 7.8}\abup{1.4}$
            & $44.1{\scriptstyle \pm 12.7}\abup{\textbf{4.0}}$ \\
            Hopper-me
            & $17.8{\scriptstyle \pm 7.9}$
            & $39.6{\scriptstyle \pm 8.2}\abup{21.8}$
            & $39.8{\scriptstyle \pm 8.3}\abup{22.0}$
            & $32.8{\scriptstyle \pm 18.5}\abup{15.0}$
            & $26.6{\scriptstyle \pm 8.6}\abup{8.8}$
            & $42.3{\scriptstyle \pm 10.8}\abup{\textbf{24.5}}$ \\
            Hopper-e
            & $23.2{\scriptstyle \pm 18.2}$
            & $57.1{\scriptstyle \pm 29.7}\abup{33.9}$
            & $55.2{\scriptstyle \pm 34.5}\abup{32.0}$
            & $16.9{\scriptstyle \pm 2.8}\abdown{6.3}$
            & $30.6{\scriptstyle \pm 15.1}\abup{7.4}$
            & $93.5{\scriptstyle \pm 13.6}\abup{\textbf{70.3}}$ \\
            \midrule
            HalfCheetah-m
            & $16.4{\scriptstyle \pm 10.2}$
            & $29.6{\scriptstyle \pm 5.6}\abup{13.2}$
            & $30.3{\scriptstyle \pm 7.1}\abup{13.9}$
            & $6.6{\scriptstyle \pm 3.0}\abdown{9.8}$
            & $28.1{\scriptstyle \pm 4.5}\abup{11.7}$
            & $31.5{\scriptstyle \pm 2.4}\abup{\textbf{15.1}}$ \\
            HalfCheetah-me
            & $15.4{\scriptstyle \pm 10.7}$
            & $19.8{\scriptstyle \pm 3.4}\abup{4.4}$
            & $28.6{\scriptstyle \pm 3.8}\abup{13.2}$
            & $3.3{\scriptstyle \pm 2.2}\abdown{12.1}$
            & $20.6{\scriptstyle \pm 3.5}\abup{5.2}$
            & $35.6{\scriptstyle \pm 2.4}\abup{\textbf{20.2}}$ \\
            HalfCheetah-e
            & $1.7{\scriptstyle \pm 3.3}$
            & $1.7{\scriptstyle \pm 1.6}$
            & $2.4{\scriptstyle \pm 1.9}\abminor{0.7}$
            & $0.8{\scriptstyle \pm 0.6}\abdown{0.9}$
            & $0.2{\scriptstyle \pm 1.8}\abdown{1.5}$
            & $3.9{\scriptstyle \pm 2.1}\abup{\textbf{2.2}}$ \\
            \midrule
            Walker2d-m
            & $7.4{\scriptstyle \pm 13.1}$
            & $24.5{\scriptstyle \pm 3.6}\abup{17.1}$
            & $20.4{\scriptstyle \pm 4.1}\abup{13.0}$
            & $8.9{\scriptstyle \pm 2.8}\abup{1.5}$
            & $8.6{\scriptstyle \pm 6.5}\abup{1.2}$
            & $41.1{\scriptstyle \pm 10.8}\abup{\textbf{33.7}}$ \\
            Walker2d-me
            & $7.9{\scriptstyle \pm 9.1}$
            & $17.8{\scriptstyle \pm 10.1}\abup{9.9}$
            & $23.3{\scriptstyle \pm 9.8}\abup{15.4}$
            & $5.9{\scriptstyle \pm 3.6}\abdown{2.0}$
            & $8.3{\scriptstyle \pm 6.1}\abminor{0.4}$
            & $37.8{\scriptstyle \pm 12.3}\abup{\textbf{29.9}}$ \\
            Walker2d-e
            & $2.7{\scriptstyle \pm 3.1}$
            & $3.9{\scriptstyle \pm 2.1}\abup{1.2}$
            & $7.4{\scriptstyle \pm 3.6}\abup{4.7}$
            & $13.8{\scriptstyle \pm 6.8}\abup{11.1}$
            & $3.9{\scriptstyle \pm 1.6}\abup{1.2}$
            & $82.3{\scriptstyle \pm 10.8}\abup{\textbf{79.5}}$ \\
            \midrule
            Average Score
            & $14.7$
            & $26.3\abup{11.6}$
            & $27.8\abup{13.1}$
            & $14.7$
            & $18.7\abup{4.0}$
            & $45.8\abup{\textbf{31.1}}$ \\
            \bottomrule
        \end{tabular}%
    }
\end{table}
\textbf{Comparing with alternative action-consistency approaches.}\quad
To further benchmark our method, we compare against two categories of representative alternatives, using TD3+BC as the shared backbone. The first category enforces action smoothness, including the penalty-based methods CAPS~\citep{mysore2021regularizing} and L2C2~\citep{kobayashi2022l2c2}, which encourage local action consistency on perturbed or interpolated states, and the architectural method LipsNet~\citep{song2023lipsnet}, which explicitly bounds the network's Lipschitz constant. The second category explicitly predicts first-order actions, for which we construct an auxiliary-head variant of TD3+BC, denoted as +Extra $\dot{a}$ head, where the actor is equipped with a dedicated prediction head, i.e., learning a policy $\pi(s) = (a, \dot{a})$, and the predicted $\dot{a}$ is also utilized during inference. All baselines follow their original configurations or exactly the same training configurations as our method.

For action-smoothness approaches, as reported in Table~\ref{tab:ablation_baselines} (+CAPS/+L2C2/+LipsNet columns), our method achieves consistent improvements across all tasks. Penalty-based methods were originally designed for online settings; in offline RL, enforcing smoothness on synthetic, unobserved states risks OOD extrapolation errors, and this degradation is particularly pronounced on high-quality (\emph{-e}) datasets, whose narrow and near-deterministic trajectories conflict with isotropic perturbations or uniform interpolations. Meanwhile, strictly constraining the network architecture (LipsNet) limits model expressivity, making it difficult to fit offline data from limited samples.

For explicit first-order action prediction, the auxiliary-head variant brings only marginal improvements over the backbone and substantially underperforms our method (+Extra $\dot{a}$ head column). We attribute this gap to two factors. First, predicting two distinct action outputs converts policy optimization into a multi-task learning problem, which introduces conflicting gradient directions between the zeroth-order and first-order action objectives and requires extra hyperparameter tuning to balance their trade-off. Second, the dedicated head doubles the output dimension and requires structural modification of the base network, losing compatibility with off-the-shelf policy architectures. In contrast, our method requires zero architectural changes, acting as a lightweight, plug-and-play module that preserves temporal consistency without sacrificing network capacity.

\begin{wrapfigure}[17]{r}{0.5\textwidth}
\vspace{-10pt}
\centering
\vspace{-4pt}
\includegraphics[width=\linewidth]{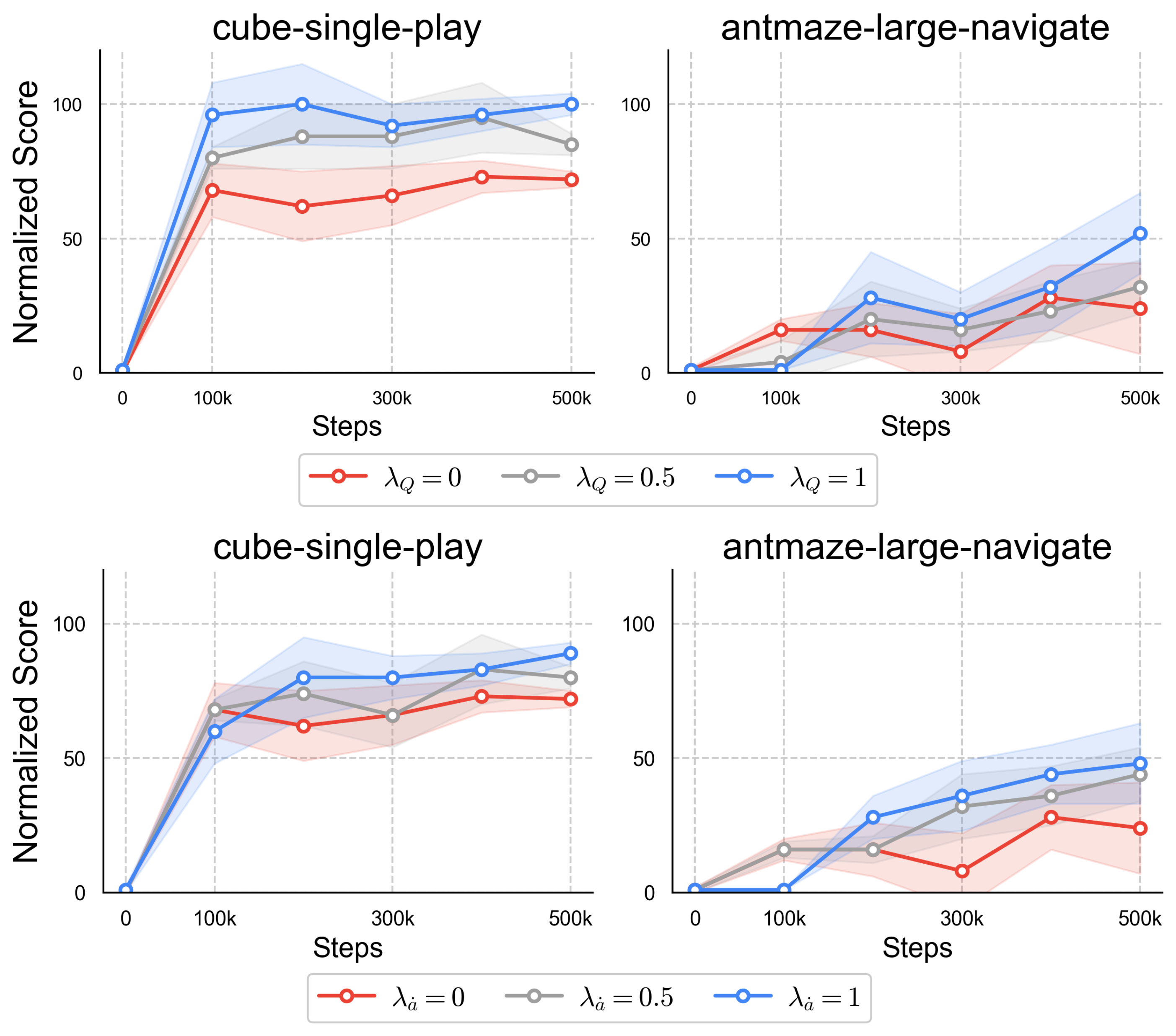}
\caption{\small Ablation study on value regularization weight $\lambda_Q$ and $1^\text{st}$-order action supervision weight $\lambda_{\dot{a}}$ of IFQL.}
\label{fig:ifql_ablation}
\vspace{-8pt}
\end{wrapfigure}

\textbf{Effectiveness of value and policy constraints.}
Figure~\ref{fig:ifql_ablation} isolates the respective contributions of the two components in our IFQL instantiation on OGBench by sweeping one weight while strictly setting the other to zero. 
The weight $\lambda_Q$ controls the first-order value-consistency regularizer; sweeping it over $\{0, 0.5, 1\}$ reveals steady performance gains, indicating that aligning the value landscape with the local state-action flow provides a highly beneficial training signal.

Similarly, increasing the first-order action supervision weight $\lambda_{\dot{a}}$ from the same grid $\{0, 0.5, 1\}$ consistently improves performance over standard IFQL. This confirms that first-order action supervision offers crucial guidance beyond mere zeroth-order action fitting, even when the value-side regularizer is completely disabled. Additional ablations are deferred to Appendix~\ref{Additional Ablation Results}.
\par

\section{Conclusion}
% \vspace{-4pt}
\label{sec:conclusion}
In this paper, we identified the neglect of higher-order action dynamics as a fundamental source of control instability and poor robustness in contemporary offline RL and IL methods. To bridge this gap, we proposed an elegant, plug-and-play learning scheme that enables simultaneous zeroth- and first-order action supervision for any off-the-shelf policy model, including deterministic, stochastic, and generative flow policies, without requiring any structural modifications. Supported by formal theoretical guarantees, our enhanced policy class also integrates seamlessly into the continuous-time HJB framework, providing additional first-order regularization on both value estimation and policy optimization. Extensive empirical evaluations confirm that our scheme yields substantial improvements in control smoothness, robustness, and out-of-distribution generalization, particularly under challenging small-dataset settings, making it a promising tool for real-world control problems.

% In this paper, we identify that the neglect of higher-order action dynamics is a fundamental source of control instability and poor robustness in contemporary offline RL and IL methods. To bridge this gap, we introduced an elegant, plug-and-play learning scheme that enables simultaneous $0^\text{th}$- and $1^\text{st}$-order action supervision for deterministic, stochastic, and generative flow policies. Supported by formal theoretical guarantees, our enhanced policy class seamlessly integrates into the continuous-time HJB framework to regularize both value estimation and policy optimization. Extensive empirical evaluations demonstrate that our method yields substantial improvements in control smoothness, performance, and out-of-distribution generalization, particularly in challenging low-data regimes.
\newpage
\section*{Acknowledgment}
This work was supported by the Wuxi Research Institute of Applied Technologies at Tsinghua University under the Grant 20242001120, Xiongan AI Institute, and Horizon Robotics.

\bibliography{neurips_2025}
\bibliographystyle{unsrt}
%%%%%%%%%%%%%%%%%%%%%%%%%%%%%%%%%%%%%%%%%%%%%%%%%%%%%%%%%%%%%%%%%%%%%%%%%%%%%%%
%%%%%%%%%%%%%%%%%%%%%%%%%%%%%%%%%%%%%%%%%%%%%%%%%%%%%%%%%%%%%%%%%%%%%%%%%%%%%%%
% Checklist
%%%%%%%%%%%%%%%%%%%%%%%%%%%%%%%%%%%%%%%%%%%%%%%%%%%%%%%%%%%%%%%%%%%%%%%%%%%%%%%
%%%%%%%%%%%%%%%%%%%%%%%%%%%%%%%%%%%%%%%%%%%%%%%%%%%%%%%%%%%%%%%%%%%%%%%%%%%%%%%

%%%%%%%%%%%%%%%%%%%%%%%%%%%%%%%%%%%%%%%%%%%%%%%%%%%%%%%%%%%%%%%%%%%%%%%%%%%%%%%
%%%%%%%%%%%%%%%%%%%%%%%%%%%%%%%%%%%%%%%%%%%%%%%%%%%%%%%%%%%%%%%%%%%%%%%%%%%%%%%
% APPENDIX
%%%%%%%%%%%%%%%%%%%%%%%%%%%%%%%%%%%%%%%%%%%%%%%%%%%%%%%%%%%%%%%%%%%%%%%%%%%%%%%
%%%%%%%%%%%%%%%%%%%%%%%%%%%%%%%%%%%%%%%%%%%%%%%%%%%%%%%%%%%%%%%%%%%%%%%%%%%%%%%
\newpage
\appendix
\vspace{-6pt}
\section{Limitations}
\label{limitaion}
\vspace{-8pt}
In this section, we summarize the limitations of our proposed method and outline 
potential directions for future research.
First, applying first-order supervision directly in raw pixel space can be computationally expensive and semantically noisy. The issue is not that higher-order supervision is incompatible with visual inputs, but that pixels are a poor coordinate system for measuring dynamics and control evolution. A practical remedy is to apply the same supervision after mapping observations into a compact state representation~\citep{hafner2020dream}, where derivatives reflect task-relevant changes rather than appearance variation. Pixel-level higher-order supervision therefore requires a good representation.
Second, our higher-order supervision relies on finite-difference estimates of local derivatives, i.e., $\dot{s}=s'-s$. With coarse sampling, these estimated tangents may provide poor approximations of the true local dynamics. This issue can be alleviated by reconstructing smoother trajectories via interpolation~\citep{deboor2001practical} or filtering~\citep{savitzky1964smoothing} before applying supervision.
\vspace{-8pt}
\section{Additional Results}
\vspace{-8pt}
\label{Additional Results}
In this section, we first present the detailed evaluation results on OGBench in Table~\ref{tab:OGbench full results}, followed by the performance scores on the D4RL benchmark under the full-dataset setting in Table~\ref{tab:d4rl_res_full}. Additionally, we report the imitation learning results in Table~\ref{tab:im_d4rl_res_10k}. Finally, Section~\ref{Additional Ablation Results} provides more ablation studies.
\vspace{-12pt}
\subsection{Additional Evaluation Results}
\label{Additional Evaluation Results}
\textbf{Full Evaluation results on OGBench.}\quad We present the comprehensive per-task results on OGBench in Table~\ref{tab:OGbench full results}. Note that the baseline performance of TD3+BC~\citep{fujimoto2021minimalist} is evaluated using the hyperparameters provided in Appendix~\ref{app:imp_details}, while the baseline results for ReBRAC, IQL, and IFQL are directly sourced from FQL~\citep{Park2025FlowQ}.
\begin{table}[H]
    \vspace{-10pt}
    \caption{\small Average normalized scores on 50 OGBench tasks over 5 random seeds. Values after $\pm$ denote standard deviations. The ${\color{mygreen}\uparrow}$ marks indicate improvements greater than $1.0$, whereas ${\color{gray}\uparrow}$ denotes minor improvements.}
    \label{tab:OGbench full results}
    \centering
        \resizebox{\textwidth}{!}{
            \begin{tabular}{lcccccccc}
                \toprule
                & \multicolumn{2}{c}{Det. Policy} & \multicolumn{2}{c}{Sto. Policy (mean sup.)} & \multicolumn{2}{c}{Sto. Policy (standard)} & \multicolumn{2}{c}{Flow Policy} \\
\cmidrule(lr){2-3} \cmidrule(lr){4-5} \cmidrule(lr){6-7} \cmidrule(lr){8-9}
                
                Task Category 
                & TD3+BC & TD3+BC+ours
                & ReBRAC & ReBRAC+ours & IQL 
                & IQL+ours & IFQL & IFQL+ours \\
                \midrule
                antmaze-large-navigate-singletask-task1-v0 ($\ast $) &$92{\scriptstyle \pm 8}$ &$99{\scriptstyle \pm 2}\,({\color{mygreen}\uparrow\! 7})$ &$91{\scriptstyle \pm 10}$ &$100{\scriptstyle \pm 1}\,({\color{mygreen}\uparrow\! 9})$  &$48{\scriptstyle \pm 9}$  &$80{\scriptstyle \pm 10}\,({\color{mygreen}\uparrow\! 32})$  &$24{\scriptstyle \pm 17}$  & $56{\scriptstyle \pm 14}\,({\color{mygreen}\uparrow\! 32})$\\
                antmaze-large-navigate-singletask-task2-v0 &$85{\scriptstyle \pm 3}$ & $91{\scriptstyle \pm 4}\,({\color{mygreen}\uparrow\! 6})$&$88{\scriptstyle \pm 4}$ &$92{\scriptstyle \pm 4}\,({\color{mygreen}\uparrow\! 4})$  &$42{\scriptstyle \pm 6}$  &$42{\scriptstyle \pm 15}$  &$8{\scriptstyle \pm 3}$  &$16{\scriptstyle \pm 11}\,({\color{mygreen}\uparrow\! 8})$  \\
                antmaze-large-navigate-singletask-task3-v0 &$72{\scriptstyle \pm 10}$ &$100{\scriptstyle \pm 0}\,({\color{mygreen}\uparrow\! 28})$ &$51{\scriptstyle \pm 18}$ &$100{\scriptstyle \pm 1}\,({\color{mygreen}\uparrow\! 49})$  &$72{\scriptstyle \pm 7}$  &$90{\scriptstyle \pm 6}\,({\color{mygreen}\uparrow\! 18})$  &$52{\scriptstyle \pm 17}$  &$72{\scriptstyle \pm 10}\,({\color{mygreen}\uparrow\! 20})$  \\
                antmaze-large-navigate-singletask-task4-v0 &$85{\scriptstyle \pm 5}$ &$99{\scriptstyle \pm 1}\,({\color{mygreen}\uparrow\! 14})$ &$84{\scriptstyle \pm 7}$ &$100{\scriptstyle \pm 1}\,({\color{mygreen}\uparrow\! 16})$  &$51{\scriptstyle \pm 9}$  &$52{\scriptstyle \pm 22}\,({\color{gray}\uparrow\! 1})$  &$18{\scriptstyle \pm 8}$  &$40{\scriptstyle \pm 17}\,({\color{mygreen}\uparrow\! 22})$  \\
                antmaze-large-navigate-singletask-task5-v0 &$82{\scriptstyle \pm 10}$ &$99{\scriptstyle \pm 0}\,({\color{mygreen}\uparrow\! 17})$ &$90{\scriptstyle \pm 2}$ &$99{\scriptstyle \pm 1}\,({\color{mygreen}\uparrow\! 9})$  &$54{\scriptstyle \pm 22}$  &$78{\scriptstyle \pm 9}\,({\color{mygreen}\uparrow\! 24})$  &$38{\scriptstyle \pm 18}$  & $48{\scriptstyle \pm 19}\,({\color{mygreen}\uparrow\! 10})$ \\
                \midrule
                antmaze-giant-navigate-singletask-task1-v0 ($\ast $) &$52{\scriptstyle \pm 20}$ &$57{\scriptstyle \pm 19}\,({\color{mygreen}\uparrow\! 5})$ &$27{\scriptstyle \pm 22}$  &$39{\scriptstyle \pm 21}\,({\color{mygreen}\uparrow\! 12})$  &$0{\scriptstyle \pm 0}$  &$0{\scriptstyle \pm 0}$  &$0{\scriptstyle \pm 0}$  & $0{\scriptstyle \pm 0}$ \\
                antmaze-giant-navigate-singletask-task2-v0 &$47{\scriptstyle \pm 30}$ &$71{\scriptstyle \pm 36}\,({\color{mygreen}\uparrow\! 24})$ &$16{\scriptstyle \pm 17}$  &$70{\scriptstyle \pm 26}\,({\color{mygreen}\uparrow\! 54})$  &$1{\scriptstyle \pm 1}$  &$2{\scriptstyle \pm 1}\,({\color{gray}\uparrow\! 1})$  &$0{\scriptstyle \pm 0}$  & $0{\scriptstyle \pm 0}$ \\
                antmaze-giant-navigate-singletask-task3-v0 &$44{\scriptstyle \pm 13}$ &$44{\scriptstyle \pm 12}$ &$34{\scriptstyle \pm 22}$  &$40{\scriptstyle \pm 15}\,({\color{mygreen}\uparrow\! 6})$  &$0{\scriptstyle \pm 0}$  &$0{\scriptstyle \pm 0}$  &$0{\scriptstyle \pm 0}$  &$0{\scriptstyle \pm 0}$  \\
                antmaze-giant-navigate-singletask-task4-v0 &$3{\scriptstyle \pm 4}$ &$63{\scriptstyle \pm 31}\,({\color{mygreen}\uparrow\! 60})$ &$5{\scriptstyle \pm 12}$  &$67{\scriptstyle \pm 33}\,({\color{mygreen}\uparrow\! 62})$  &$0{\scriptstyle \pm 0}$  &$1{\scriptstyle \pm 1}\,({\color{gray}\uparrow\! 1})$  &$0{\scriptstyle \pm 0}$  &$0{\scriptstyle \pm 0}$  \\
                antmaze-giant-navigate-singletask-task5-v0 &$65{\scriptstyle \pm 7}$ &$92{\scriptstyle \pm 11}\,({\color{mygreen}\uparrow\! 27})$ &$49{\scriptstyle \pm 22}$  &$96{\scriptstyle \pm 3}\,({\color{mygreen}\uparrow\! 47})$  &$19{\scriptstyle \pm 7}$  &$41{\scriptstyle \pm 6}\,({\color{mygreen}\uparrow\! 22})$  &$13{\scriptstyle \pm 9}$  & $21{\scriptstyle \pm 17}\,({\color{mygreen}\uparrow\! 8})$ \\
                \midrule
                humanoidmaze-medium-navigate-singletask-task1-v0 ($\ast $) &$24{\scriptstyle \pm 11}$ &$26{\scriptstyle \pm 4}\,({\color{mygreen}\uparrow\! 2})$ &$16{\scriptstyle \pm 9}$   &$36{\scriptstyle \pm 19}\,({\color{mygreen}\uparrow\! 20})$  &$32{\scriptstyle \pm 7}$  &$34{\scriptstyle \pm 5}\,({\color{mygreen}\uparrow\! 2})$  &$69{\scriptstyle \pm 19}$  & $71{\scriptstyle \pm 16}\,({\color{mygreen}\uparrow\! 2})$  \\
                humanoidmaze-medium-navigate-singletask-task2-v0 &$45{\scriptstyle \pm 13}$ &$49{\scriptstyle \pm 9}\,({\color{mygreen}\uparrow\! 4})$ &$18{\scriptstyle \pm 16}$   &$68{\scriptstyle \pm 24}\,({\color{mygreen}\uparrow\! 50})$  &$41{\scriptstyle \pm 9}$  &$51{\scriptstyle \pm 5}\,({\color{mygreen}\uparrow\! 10})$  &$85{\scriptstyle \pm 11}$  & $91{\scriptstyle \pm 17}\,({\color{mygreen}\uparrow\! 6})$ \\
                humanoidmaze-medium-navigate-singletask-task3-v0 &$40{\scriptstyle \pm 26}$ &$60{\scriptstyle \pm 19}\,({\color{mygreen}\uparrow\! 20})$ &$36{\scriptstyle \pm 13}$   &$66{\scriptstyle \pm 22}\,({\color{mygreen}\uparrow\! 30})$  &$25{\scriptstyle \pm 5}$  &$27{\scriptstyle \pm 6}\,({\color{mygreen}\uparrow\! 2})$  &$49{\scriptstyle \pm 49}$  & $56{\scriptstyle \pm 32}\,({\color{mygreen}\uparrow\! 7})$ \\
                humanoidmaze-medium-navigate-singletask-task4-v0 &$30{\scriptstyle \pm 20}$ &$44{\scriptstyle \pm 9}\,({\color{mygreen}\uparrow\! 14})$ &$15{\scriptstyle \pm 16}$   &$30{\scriptstyle \pm 21}\,({\color{mygreen}\uparrow\! 15})$  &$0{\scriptstyle \pm 1}$  &$0{\scriptstyle \pm 0}$  &$1{\scriptstyle \pm 1}$  & $1{\scriptstyle \pm 1}$ \\
                humanoidmaze-medium-navigate-singletask-task5-v0 &$31{\scriptstyle \pm 13}$ &$31{\scriptstyle \pm 7}$ &$24{\scriptstyle \pm 20}$   &$78{\scriptstyle \pm 16}\,({\color{mygreen}\uparrow\! 54})$  &$66{\scriptstyle \pm 4}$  &$66{\scriptstyle \pm 4}$  &$98{\scriptstyle \pm 2}$  & $100{\scriptstyle \pm 1}\,({\color{mygreen}\uparrow\! 2})$ \\
                \midrule
                humanoidmaze-large-navigate-singletask-task1-v0 ($\ast $) &$3{\scriptstyle \pm 3}$ &$3{\scriptstyle \pm 2}$ &$2{\scriptstyle \pm 1}$  &$7{\scriptstyle \pm 7}\,({\color{mygreen}\uparrow\! 5})$  &$3{\scriptstyle \pm 1}$  &$4{\scriptstyle \pm 1}\,({\color{gray}\uparrow\! 1})$  &$6{\scriptstyle \pm 2}$  & $8{\scriptstyle \pm 4}\,({\color{mygreen}\uparrow\! 2})$ \\
                humanoidmaze-large-navigate-singletask-task2-v0 &$0{\scriptstyle \pm 0}$ &$0{\scriptstyle \pm 0}$ &$0{\scriptstyle \pm 0}$  &$0{\scriptstyle \pm 0}$  &$0{\scriptstyle \pm 0}$  &$0{\scriptstyle \pm 0}$  &$0{\scriptstyle \pm 0}$  &$0{\scriptstyle \pm 0}$  \\
                humanoidmaze-large-navigate-singletask-task3-v0 &$8{\scriptstyle \pm 4}$ &$17{\scriptstyle \pm 6}\,({\color{mygreen}\uparrow\! 9})$ &$8{\scriptstyle \pm 4}$  &$11{\scriptstyle \pm 6}\,({\color{mygreen}\uparrow\! 3})$  &$7{\scriptstyle \pm 3}$  &$7{\scriptstyle \pm 1}$  &$48{\scriptstyle \pm 10}$  & $52{\scriptstyle \pm 16}\,({\color{mygreen}\uparrow\! 4})$ \\
                humanoidmaze-large-navigate-singletask-task4-v0 &$1{\scriptstyle \pm 1}$ &$2{\scriptstyle \pm 1}\,({\color{gray}\uparrow\! 1})$ &$1{\scriptstyle \pm 1}$  &$2{\scriptstyle \pm 1}\,({\color{gray}\uparrow\! 1})$  &$1{\scriptstyle \pm 0}$  &$2{\scriptstyle \pm 1}\,({\color{gray}\uparrow\! 1})$  &$1{\scriptstyle \pm 1}$  &$1{\scriptstyle \pm 1}$  \\
                humanoidmaze-large-navigate-singletask-task5-v0 &$0{\scriptstyle \pm 0}$ &$1{\scriptstyle \pm 1}\,({\color{gray}\uparrow\! 1})$ &$2{\scriptstyle \pm 2}$  &$2{\scriptstyle \pm 1}$  &$1{\scriptstyle \pm 1}$  &$3{\scriptstyle \pm 1}\,({\color{mygreen}\uparrow\! 2})$  &$0{\scriptstyle \pm 0}$  &$0{\scriptstyle \pm 0}$  \\
                \midrule
                antsoccer-arena-navigate-singletask-task1-v0 &$9{\scriptstyle \pm 3}$ &$11{\scriptstyle \pm 5}\,({\color{mygreen}\uparrow\! 2})$ &$0{\scriptstyle \pm 0}$  &$13{\scriptstyle \pm 6}\,({\color{mygreen}\uparrow\! 13})$  &$14{\scriptstyle \pm 5}$  &$36{\scriptstyle \pm 10}\,({\color{mygreen}\uparrow\! 22})$  &$61{\scriptstyle \pm 26}$  & $68{\scriptstyle \pm 21}\,({\color{mygreen}\uparrow\! 7})$ \\
                antsoccer-arena-navigate-singletask-task2-v0 &$19{\scriptstyle \pm 8}$ &$19{\scriptstyle \pm 2}$ &$0{\scriptstyle \pm 1}$  &$13{\scriptstyle \pm 4}\,({\color{mygreen}\uparrow\! 13})$  &$17{\scriptstyle \pm 7}$  &$37{\scriptstyle \pm 5}\,({\color{mygreen}\uparrow\! 20})$  &$75{\scriptstyle \pm 3}$  & $83{\scriptstyle \pm 12}\,({\color{mygreen}\uparrow\! 8})$ \\
                antsoccer-arena-navigate-singletask-task3-v0 &$25{\scriptstyle \pm 8}$ &$27{\scriptstyle \pm 10}\,({\color{mygreen}\uparrow\! 2})$ &$0{\scriptstyle \pm 0}$  &$32{\scriptstyle \pm 11}\,({\color{mygreen}\uparrow\! 32})$  &$6{\scriptstyle \pm 4}$  &$19{\scriptstyle \pm 7}\,({\color{mygreen}\uparrow\! 13})$  &$14{\scriptstyle \pm 22}$  &$22{\scriptstyle \pm 18}\,({\color{mygreen}\uparrow\! 8})$  \\
                antsoccer-arena-navigate-singletask-task4-v0 ($\ast $) &$1{\scriptstyle \pm 1}$ &$1{\scriptstyle \pm 1}$  &$0{\scriptstyle \pm 0}$  &$20{\scriptstyle \pm 4}\,({\color{mygreen}\uparrow\! 20})$  &$3{\scriptstyle \pm 2}$&$9{\scriptstyle \pm 3}\,({\color{mygreen}\uparrow\! 6})$  &$16{\scriptstyle \pm 9}$  &$16{\scriptstyle \pm 8}$  \\
                antsoccer-arena-navigate-singletask-task5-v0 &$12{\scriptstyle \pm 7}$ &$20{\scriptstyle \pm 9}\,({\color{mygreen}\uparrow\! 8})$ &$0{\scriptstyle \pm 0}$  &$6{\scriptstyle \pm 5}\,({\color{mygreen}\uparrow\! 6})$  &$2{\scriptstyle \pm 2}$  &$11{\scriptstyle \pm 7}\,({\color{mygreen}\uparrow\! 9})$  &$0{\scriptstyle \pm 1}$  & $8{\scriptstyle \pm 7}\,({\color{mygreen}\uparrow\! 8})$ \\
                \midrule
                cube-single-play-singletask-task1-v0 &$89{\scriptstyle \pm 3}$ &$94{\scriptstyle \pm 2}\,({\color{mygreen}\uparrow\! 5})$ &$89{\scriptstyle \pm 5}$  &$94{\scriptstyle \pm 2}\,({\color{mygreen}\uparrow\! 5})$  &$88{\scriptstyle \pm 3}$  &$90{\scriptstyle \pm 5}\,({\color{mygreen}\uparrow\! 2})$  &$79{\scriptstyle \pm 4}$  & $96{\scriptstyle \pm 9}\,({\color{mygreen}\uparrow\! 17})$ \\
                cube-single-play-singletask-task2-v0 ($\ast $) &$87{\scriptstyle \pm 5}$ &$88{\scriptstyle \pm 2}\,({\color{gray}\uparrow\! 1})$ &$92{\scriptstyle \pm 4}$  &$93{\scriptstyle \pm 4}\,({\color{gray}\uparrow\! 1})$  &$85{\scriptstyle \pm 8}$  &$93{\scriptstyle \pm 3}\,({\color{mygreen}\uparrow\! 8})$  &$73{\scriptstyle \pm 3}$  & $100{\scriptstyle \pm 1}\,({\color{mygreen}\uparrow\! 27})$ \\
                cube-single-play-singletask-task3-v0 &$93{\scriptstyle \pm 3}$ &$96{\scriptstyle \pm 3}\,({\color{mygreen}\uparrow\! 3})$ &$93{\scriptstyle \pm 3}$  &$94{\scriptstyle \pm 2}\,({\color{gray}\uparrow\! 1})$  &$91{\scriptstyle \pm 5}$  &$95{\scriptstyle \pm 5}\,({\color{mygreen}\uparrow\! 4})$  &$88{\scriptstyle \pm 4}$  & $100{\scriptstyle \pm 1}\,({\color{mygreen}\uparrow\! 12})$  \\
                cube-single-play-singletask-task4-v0 &$87{\scriptstyle \pm 4}$ &$89{\scriptstyle \pm 3}\,({\color{mygreen}\uparrow\! 2})$ &$92{\scriptstyle \pm 3}$  &$92{\scriptstyle \pm 4}$  &$73{\scriptstyle \pm 6}$  &$90{\scriptstyle \pm 5}\,({\color{mygreen}\uparrow\! 17})$  &$79{\scriptstyle \pm 6}$  & $88{\scriptstyle \pm 16}\,({\color{mygreen}\uparrow\! 9})$ \\
                cube-single-play-singletask-task5-v0 &$85{\scriptstyle \pm 5}$ &$85{\scriptstyle \pm 1}$ &$87{\scriptstyle \pm 8}$  &$90{\scriptstyle \pm 4}\,({\color{mygreen}\uparrow\! 3})$  &$78{\scriptstyle \pm 9}$  &$90{\scriptstyle \pm 6}\,({\color{mygreen}\uparrow\! 12})$  &$77{\scriptstyle \pm 7}$  & $92{\scriptstyle \pm 10}\,({\color{mygreen}\uparrow\! 15})$ \\
                \midrule
                cube-double-play-singletask-task1-v0 &$30{\scriptstyle \pm 4}$ &$46{\scriptstyle \pm 6}\,({\color{mygreen}\uparrow\! 16})$ &$45{\scriptstyle \pm 6}$  &$45{\scriptstyle \pm 6}$  &$27{\scriptstyle \pm 5}$  &$33{\scriptstyle \pm 7}\,({\color{mygreen}\uparrow\! 6})$  &$35{\scriptstyle \pm 9}$  & $76{\scriptstyle \pm 8}\,({\color{mygreen}\uparrow\! 41})$ \\
                cube-double-play-singletask-task2-v0 ($\ast $) &$6{\scriptstyle \pm 1}$ &$9{\scriptstyle \pm 1}\,({\color{mygreen}\uparrow\! 3})$ &$7{\scriptstyle \pm 3}$  &$32{\scriptstyle \pm 7}\,({\color{mygreen}\uparrow\! 25})$  &$1{\scriptstyle \pm 1}$  &$2{\scriptstyle \pm 1}\,({\color{gray}\uparrow\! 1})$  &$9{\scriptstyle \pm 5}$  & $40{\scriptstyle \pm 22}\,({\color{mygreen}\uparrow\! 31})$ \\
                cube-double-play-singletask-task3-v0 &$1{\scriptstyle \pm 1}$ &$2{\scriptstyle \pm 1}\,({\color{gray}\uparrow\! 1})$ &$4{\scriptstyle \pm 1}$  &$17{\scriptstyle \pm 5}\,({\color{mygreen}\uparrow\! 13})$  &$0{\scriptstyle \pm 0}$  &$1{\scriptstyle \pm 1}\,({\color{gray}\uparrow\! 1})$  &$8{\scriptstyle \pm 5}$  & $25{\scriptstyle \pm 13}\,({\color{mygreen}\uparrow\! 17})$ \\
                cube-double-play-singletask-task4-v0 &$0{\scriptstyle \pm 1}$ &$1{\scriptstyle \pm 1}\,({\color{gray}\uparrow\! 1})$ &$1{\scriptstyle \pm 1}$  &$3{\scriptstyle \pm 1}\,({\color{mygreen}\uparrow\! 2})$  &$0{\scriptstyle \pm 0}$  &$0{\scriptstyle \pm 0}$  &$1{\scriptstyle \pm 1}$  & $12{\scriptstyle \pm 10}\,({\color{mygreen}\uparrow\! 11})$ \\
                cube-double-play-singletask-task5-v0 &$1{\scriptstyle \pm 1}$ &$1{\scriptstyle \pm 1}$ &$4{\scriptstyle \pm 2}$  &$22{\scriptstyle \pm 7}\,({\color{mygreen}\uparrow\! 18})$  &$4{\scriptstyle \pm 3}$  &$4{\scriptstyle \pm 2}$  &$17{\scriptstyle \pm 6}$  & $55{\scriptstyle \pm 21}\,({\color{mygreen}\uparrow\! 38})$ \\
                \midrule
                scene-play-singletask-task1-v0 &$95{\scriptstyle \pm 3}$ &$96{\scriptstyle \pm 1}\,({\color{gray}\uparrow\! 1})$ &$95{\scriptstyle \pm 2}$  &$95{\scriptstyle \pm 3}$  &$94{\scriptstyle \pm 3}$  &$95{\scriptstyle \pm 2}\,({\color{gray}\uparrow\! 1})$  &$98{\scriptstyle \pm 3}$  & $99{\scriptstyle \pm 2}\,({\color{gray}\uparrow\! 1})$ \\
                scene-play-singletask-task2-v0 ($\ast $) &$54{\scriptstyle \pm 10}$ &$58{\scriptstyle \pm 4}\,({\color{mygreen}\uparrow\! 4})$ &$50{\scriptstyle \pm 13}$  &$60{\scriptstyle \pm 13}\,({\color{mygreen}\uparrow\! 10})$  &$12{\scriptstyle \pm 3}$  &$17{\scriptstyle \pm 6}\,({\color{mygreen}\uparrow\! 5})$  &$0{\scriptstyle \pm 0}$  & $88{\scriptstyle \pm 17}\,({\color{mygreen}\uparrow\! 88})$ \\
                scene-play-singletask-task3-v0 &$68{\scriptstyle \pm 10}$ &$78{\scriptstyle \pm 7}\,({\color{mygreen}\uparrow\! 10})$ &$55{\scriptstyle \pm 16}$  &$75{\scriptstyle \pm 10}\,({\color{mygreen}\uparrow\! 20})$  &$32{\scriptstyle \pm 7}$  &$52{\scriptstyle \pm 9}\,({\color{mygreen}\uparrow\! 20})$  &$54{\scriptstyle \pm 19}$  &$92{\scriptstyle \pm 17}\,({\color{mygreen}\uparrow\! 38})$  \\
                scene-play-singletask-task4-v0 &$4{\scriptstyle \pm 3}$ &$4{\scriptstyle \pm 2}$ &$3{\scriptstyle \pm 3}$  &$5{\scriptstyle \pm 3}\,({\color{mygreen}\uparrow\! 2})$  &$0{\scriptstyle \pm 1}$  &$1{\scriptstyle \pm 1}\,({\color{gray}\uparrow\! 1})$  &$0{\scriptstyle \pm 0}$  &$0{\scriptstyle \pm 0}$  \\
                scene-play-singletask-task5-v0 &$0{\scriptstyle \pm 0}$ &$0{\scriptstyle \pm 0}$ &$0{\scriptstyle \pm 0}$  &$0{\scriptstyle \pm 0}$  &$0{\scriptstyle \pm 0}$  &$0{\scriptstyle \pm 0}$  &$0{\scriptstyle \pm 0}$  &$0{\scriptstyle \pm 0}$  \\
                \midrule
                puzzle-3x3-play-singletask-task1-v0 &$96{\scriptstyle \pm 3}$ &$98{\scriptstyle \pm 3}\,({\color{mygreen}\uparrow\! 2})$ &$97{\scriptstyle \pm 4}$  &$98{\scriptstyle \pm 3}\,({\color{gray}\uparrow\! 1})$  &$33{\scriptstyle \pm 6}$  &$39{\scriptstyle \pm 17}\,({\color{mygreen}\uparrow\! 6})$  &$94{\scriptstyle \pm 3}$  &$100{\scriptstyle \pm 1}\,({\color{mygreen}\uparrow\! 6})$  \\
                puzzle-3x3-play-singletask-task2-v0 &$13{\scriptstyle \pm 9}$ &$14{\scriptstyle \pm 4}\,({\color{gray}\uparrow\! 1})$ &$1{\scriptstyle \pm 1}$  &$4{\scriptstyle \pm 2}\,({\color{mygreen}\uparrow\! 3})$  &$4{\scriptstyle \pm 3}$  &$9{\scriptstyle \pm 6}\,({\color{mygreen}\uparrow\! 5})$  &$1{\scriptstyle \pm 2}$  & $60{\scriptstyle \pm 32}\,({\color{mygreen}\uparrow\! 59})$ \\
                puzzle-3x3-play-singletask-task3-v0 &$5{\scriptstyle \pm 6}$ &$6{\scriptstyle \pm 1}\,({\color{gray}\uparrow\! 1})$ &$3{\scriptstyle \pm 1}$  &$4{\scriptstyle \pm 4}\,({\color{gray}\uparrow\! 1})$  &$3{\scriptstyle \pm 2}$  &$3{\scriptstyle \pm 1}$  &$0{\scriptstyle \pm 0}$  & $28{\scriptstyle \pm 17}\,({\color{mygreen}\uparrow\! 28})$ \\
                puzzle-3x3-play-singletask-task4-v0 ($\ast $) &$7{\scriptstyle \pm 4}$ &$11{\scriptstyle \pm 4}\,({\color{mygreen}\uparrow\! 4})$ &$2{\scriptstyle \pm 1}$  &$4{\scriptstyle \pm 2}\,({\color{mygreen}\uparrow\! 2})$  &$2{\scriptstyle \pm 1}$  &$2{\scriptstyle \pm 1}$  &$0{\scriptstyle \pm 0}$  & $16{\scriptstyle \pm 12}\,({\color{mygreen}\uparrow\! 16})$ \\
                puzzle-3x3-play-singletask-task5-v0 &$14{\scriptstyle \pm 7}$ &$15{\scriptstyle \pm 2}\,({\color{gray}\uparrow\! 1})$ &$5{\scriptstyle \pm 3}$  &$11{\scriptstyle \pm 6}\,({\color{mygreen}\uparrow\! 6})$  &$3{\scriptstyle \pm 2}$  &$4{\scriptstyle \pm 2}\,({\color{gray}\uparrow\! 1})$  &$0{\scriptstyle \pm 0}$  & $12{\scriptstyle \pm 11}\,({\color{mygreen}\uparrow\! 12})$ \\
                \midrule
                puzzle-4x4-play-singletask-task1-v0 &$1{\scriptstyle \pm 2}$ &$14{\scriptstyle \pm 2}\,({\color{mygreen}\uparrow\! 13})$ &$26{\scriptstyle \pm 4}$  &$33{\scriptstyle \pm 9}\,({\color{mygreen}\uparrow\! 7})$  &$12{\scriptstyle \pm 2}$  &$20{\scriptstyle \pm 6}\,({\color{mygreen}\uparrow\! 8})$  &$49{\scriptstyle \pm 9}$  &$53{\scriptstyle \pm 6}\,({\color{mygreen}\uparrow\! 4})$  \\
                puzzle-4x4-play-singletask-task2-v0 &$7{\scriptstyle \pm 3}$ &$9{\scriptstyle \pm 2}\,({\color{mygreen}\uparrow\! 2})$ &$12{\scriptstyle \pm 4}$  &$13{\scriptstyle \pm 7}\,({\color{gray}\uparrow\! 1})$  &$7{\scriptstyle \pm 4}$  &$10{\scriptstyle \pm 3}\,({\color{mygreen}\uparrow\! 3})$  &$4{\scriptstyle \pm 4}$  & $12{\scriptstyle \pm 9}\,({\color{mygreen}\uparrow\! 8})$ \\
                puzzle-4x4-play-singletask-task3-v0 &$3{\scriptstyle \pm 2}$ &$8{\scriptstyle \pm 2}\,({\color{mygreen}\uparrow\! 5})$ &$15{\scriptstyle \pm 3}$  &$19{\scriptstyle \pm 6}\,({\color{mygreen}\uparrow\! 4})$  &$9{\scriptstyle \pm 3}$  &$14{\scriptstyle \pm 6}\,({\color{mygreen}\uparrow\! 5})$  &$50{\scriptstyle \pm 14}$  & $84{\scriptstyle \pm 16}\,({\color{mygreen}\uparrow\! 34})$ \\
                puzzle-4x4-play-singletask-task4-v0 ($\ast $) &$5{\scriptstyle \pm 2}$ &$7{\scriptstyle \pm 2}\,({\color{mygreen}\uparrow\! 2})$ &$10{\scriptstyle \pm 3}$  &$10{\scriptstyle \pm 4}$  &$5{\scriptstyle \pm 2}$  &$10{\scriptstyle \pm 2}\,({\color{mygreen}\uparrow\! 5})$  &$21{\scriptstyle \pm 11}$  & $30{\scriptstyle \pm 17}\,({\color{mygreen}\uparrow\! 9})$ \\
                puzzle-4x4-play-singletask-task5-v0 &$6{\scriptstyle \pm 1}$ &$6{\scriptstyle \pm 1}$ &$7{\scriptstyle \pm 3}$  &$8{\scriptstyle \pm 4}\,({\color{gray}\uparrow\! 1})$  &$4{\scriptstyle \pm 1}$  &$8{\scriptstyle \pm 5}\,({\color{mygreen}\uparrow\! 4})$  &$2{\scriptstyle \pm 2}$  & $21{\scriptstyle \pm 14}\,({\color{mygreen}\uparrow\! 19})$ \\
                \bottomrule
            \end{tabular}
            }
\vspace{-20pt}
\end{table}

The empirical results show that integrating our first-order action supervision yields consistent performance improvements across a diverse range of offline policy optimization paradigms. Notably, our scheme can improve policy learning in cases where baseline algorithms fundamentally fail. On several complex, sparse-reward tasks, baselines collapse and obtain a score of exactly $0$, indicating that they cannot extract any meaningful learning signal from the offline dataset. Equipping the same algorithm with our first-order supervision lifts these scores from $0$ to clearly non-trivial levels. For example, on \texttt{scene-play-singletask-task2-v0}, IFQL completely fails ($0$), whereas IFQL with our first-order supervision reaches $88$, a near-expert level. The same breakthrough is observed on the \texttt{puzzle-3x3-play} domain (e.g., \texttt{task3-v0} and \texttt{task4-v0}), where the baselines obtain no useful learning signal (scoring $0$), while the augmented policies acquire the temporal consistency required to achieve non-trivial success rates. Beyond raising scores on already-learnable tasks, our scheme thus expands the set of tasks on which offline policy learning is feasible at all.

The improvements are also consistent across all three policy paradigms. On high-dimensional dexterous manipulation tasks such as \texttt{cube-double-play-singletask-task2-v0}, our scheme raises the score of IFQL from $9$ to $40$. On locomotion and navigation tasks such as \texttt{antmaze-giant-navigate-singletask-task5-v0}, integrating our scheme into ReBRAC further raises the score from $49$ to $96$, approaching expert-level proficiency with noticeably reduced seed-level variance. Together, these results highlight the broad applicability of our scheme and indicate that supervising higher-order action dynamics is crucial for stable offline policy learning and reliable closed-loop execution.

\begin{table}[t]
    \caption{\small Average normalized scores on D4RL MuJoCo locomotion tasks under full-datasize. Standard deviations are reported after the $\pm$ sign. The ${\color{mygreen}\uparrow}$ symbol indicates a performance increase greater than $1.0$, whereas ${\color{gray}\uparrow}$ denotes minor improvements. All results and standard deviations ($\pm$) are averaged across 5 random seeds.}
\label{tab:d4rl_res_full}
\centering
\resizebox{\textwidth}{!}{%
\begin{tabular}{lcccccccc}
\toprule
& \multicolumn{2}{c}{Det. Policy} & \multicolumn{2}{c}{Sto. Policy (mean sup.)} & \multicolumn{2}{c}{Sto. Policy (standard)} & \multicolumn{2}{c}{Flow Policy} \\
\cmidrule(lr){2-3} \cmidrule(lr){4-5} \cmidrule(lr){6-7} \cmidrule(lr){8-9}
Task Category  & TD3+BC        & TD3+BC+ours          & ReBRAC            & ReBRAC+ours        & IQL            & IQL+ours                  & IFQL & IFQL+ours             \\ 
\midrule
Hopper-m & $55.5{\scriptstyle \pm 7.7}$ & $58.9{\scriptstyle \pm 7.4}({\color{mygreen}\uparrow\! 3.4})$ & $84.9{\scriptstyle \pm 2.6}$ & $85.7{\scriptstyle \pm 1.3}({\color{gray}\uparrow\! 0.8})$  & $67.3{\scriptstyle \pm 5.7}$ & $68.8{\scriptstyle \pm 5.5}({\color{mygreen}\uparrow\! 1.5})$ & $89.5{\scriptstyle \pm 5.7}$ & $94.2{\scriptstyle \pm 7.2}({\color{mygreen}\uparrow\! 4.7})$ \\
Hopper-mr & $57.9{\scriptstyle \pm 17.6}$ & $58.4{\scriptstyle \pm 9.2}({\color{gray}\uparrow\! 0.5})$ & $81.2{\scriptstyle \pm 25.1}$ & $94.3{\scriptstyle \pm 10.1}({\color{mygreen}\uparrow\! 13.1})$ & $91.7{\scriptstyle \pm 8.6}$ & $94.2{\scriptstyle \pm 5.5}({\color{mygreen}\uparrow\! 2.5})$ & $96.2{\scriptstyle \pm 4.7}$ & $96.3{\scriptstyle \pm 2.3}({\color{gray}\uparrow\! 0.1})$ \\
Hopper-me & $89.8{\scriptstyle \pm 15.2}$ & $89.9{\scriptstyle \pm 8.4}({\color{gray}\uparrow\! 0.1})$ & $106.4{\scriptstyle \pm 10.7}$ & $106.6{\scriptstyle \pm 5.4}({\color{gray}\uparrow\! 0.2})$ & $91.5{\scriptstyle \pm 14.3}$ & $109.2{\scriptstyle \pm 9.8}({\color{mygreen}\uparrow\! 17.7})$ & $77.1{\scriptstyle \pm 25.5}$ & $89.8{\scriptstyle \pm 8.7}({\color{mygreen}\uparrow\! 12.7})$ \\
Hopper-e & $108.3{\scriptstyle \pm 6.7}$ & $110.5{\scriptstyle \pm 0.2}({\color{mygreen}\uparrow\! 2.2})$ & $108.5{\scriptstyle \pm 4.1}$ & $108.6{\scriptstyle \pm 3.8}({\color{gray}\uparrow\! 0.1})$ & $99.3{\scriptstyle \pm 9.2}$ & $106.5{\scriptstyle \pm 5.1}({\color{mygreen}\uparrow\! 7.2})$ & $96.8{\scriptstyle \pm 12.2}$ & $106.2{\scriptstyle \pm 8.6}({\color{mygreen}\uparrow\! 9.4})$ \\
\midrule
Halfcheetah-m & $47.8{\scriptstyle \pm 0.4}$ & $47.9{\scriptstyle \pm 0.4}({\color{gray}\uparrow\! 0.1})$ & $55.4{\scriptstyle \pm 1.0}$ & $56.9{\scriptstyle \pm 1.1}({\color{mygreen}\uparrow\! 1.5})$ & $47.4{\scriptstyle \pm 0.4}$ & $48.6{\scriptstyle \pm 0.2}({\color{mygreen}\uparrow\! 1.2})$ & $51.1{\scriptstyle \pm 0.6}$ & $52.8{\scriptstyle \pm 0.3}({\color{mygreen}\uparrow\! 1.7})$ \\
Halfcheetah-mr & $41.6{\scriptstyle \pm 0.3}$ & $42.3{\scriptstyle \pm 0.2}({\color{gray}\uparrow\! 0.7})$ & $46.6{\scriptstyle \pm 0.3}$ & $46.7{\scriptstyle \pm 0.4}({\color{gray}\uparrow\! 0.1})$ & $41.3{\scriptstyle \pm 2.3}$ & $41.4{\scriptstyle \pm 2.0}({\color{gray}\uparrow\! 0.1})$ & $42.5{\scriptstyle \pm 0.4}$ & $43.6{\scriptstyle \pm 1.3}({\color{mygreen}\uparrow\! 1.1})$ \\
Halfcheetah-me & $88.1{\scriptstyle \pm 5.1}$ & $89.5{\scriptstyle \pm 4.7}({\color{mygreen}\uparrow\! 1.4})$ & $96.2{\scriptstyle \pm 1.6}$ & $97.4{\scriptstyle \pm 1.4}({\color{mygreen}\uparrow\! 1.2})$ & $86.7{\scriptstyle \pm 3.4}$ & $87.4{\scriptstyle \pm 2.6}({\color{gray}\uparrow\! 0.7})$ & $88.4{\scriptstyle \pm 2.2}$ & $89.5{\scriptstyle \pm 1.6}({\color{mygreen}\uparrow\! 1.1})$ \\
Halfcheetah-e & $90.1{\scriptstyle \pm 0.3}$ & $93.4{\scriptstyle \pm 0.4}({\color{mygreen}\uparrow\! 3.3})$ & $91.3{\scriptstyle \pm 1.0}$ &$91.7{\scriptstyle \pm 0.3}({\color{gray}\uparrow\! 0.4})$  & $89.6{\scriptstyle \pm 7.8}$ & $91.4{\scriptstyle \pm 1.2}({\color{mygreen}\uparrow\! 1.8})$ & $91.2{\scriptstyle \pm 1.5}$ & $91.6{\scriptstyle \pm 1.7}({\color{gray}\uparrow\! 0.4})$ \\
\midrule
Walker2d-m & $79.7{\scriptstyle \pm 4.1}$ & $80.1{\scriptstyle \pm 1.6}({\color{gray}\uparrow\! 0.4})$ & $76.7{\scriptstyle \pm 9.9}$ & $82.6{\scriptstyle \pm 2.2}({\color{mygreen}\uparrow\! 5.9})$ & $71.5{\scriptstyle \pm 10.8}$ & $78.8{\scriptstyle \pm 5.0}({\color{mygreen}\uparrow\! 7.3})$ & $79.5{\scriptstyle \pm 4.1}$ & $82.6{\scriptstyle \pm 4.4}({\color{mygreen}\uparrow\! 3.1})$ \\
Walker2d-mr & $67.2{\scriptstyle \pm 16.4}$ & $73.3{\scriptstyle \pm 4.1}({\color{mygreen}\uparrow\! 6.1})$ & $70.2{\scriptstyle \pm 11.5}$ & $83.6{\scriptstyle \pm 7.1}({\color{mygreen}\uparrow\! 13.4})$ & $71.5{\scriptstyle \pm 8.2}$ & $73.9{\scriptstyle \pm 5.5}({\color{mygreen}\uparrow\! 2.4})$ & $72.5{\scriptstyle \pm 19.1}$ & $83.1{\scriptstyle \pm 7.1}({\color{mygreen}\uparrow\! 10.6})$ \\
Walker2d-me & $108.8{\scriptstyle \pm 0.5}$ & $108.9{\scriptstyle \pm 0.3}({\color{gray}\uparrow\! 0.1})$ & $110.3{\scriptstyle \pm 1.4}$ & $110.4{\scriptstyle \pm 0.1}({\color{gray}\uparrow\! 0.1})$ & $108.4{\scriptstyle \pm 0.7}$ & $108.5{\scriptstyle \pm 0.5}({\color{gray}\uparrow\! 0.1})$ & $106.1{\scriptstyle \pm 1.0}$ & $107.5{\scriptstyle \pm 0.7}({\color{mygreen}\uparrow\! 1.4})$ \\
Walker2d-e & $108.6{\scriptstyle \pm 0.2}$ & $109.3{\scriptstyle \pm 0.1}({\color{gray}\uparrow\! 0.7})$ & $102.2{\scriptstyle \pm 1.5}$ & $109.7{\scriptstyle \pm 0.5}({\color{mygreen}\uparrow\! 7.5})$ & $110.3{\scriptstyle \pm 1.7}$ & $110.4{\scriptstyle \pm 0.5}({\color{gray}\uparrow\! 0.1})$ & $102.7{\scriptstyle \pm 11.5}$ & $106.4{\scriptstyle \pm 0.5}({\color{mygreen}\uparrow\! 3.7})$ \\
\bottomrule
\end{tabular}
}
\end{table}
\begin{table}[t]
    \vspace{-13pt}
    \caption{\small Imitation learning performance on D4RL MuJoCo locomotion tasks with 10k datasets. Average normalized scores across deterministic, stochastic, and flow-based policy paradigms. The ${\color{mygreen}\uparrow}$ indicates a performance increase $> 1.0$, whereas ${\color{gray}\uparrow}$ represents marginal gains. All results and standard deviations ($\pm$) are averaged across 5 random seeds.}
\label{tab:im_d4rl_res_10k}
\centering
\resizebox{0.85\textwidth}{!}{% 稍微放大了一点比例到 0.75，防止拥挤
\begin{tabular}{lccccccc}
\toprule
& \multicolumn{2}{c}{Det. Policy} & \multicolumn{2}{c}{Sto. Policy (standard)} & \multicolumn{2}{c}{Flow Policy} \\
\cmidrule(lr){2-3} \cmidrule(lr){4-5} \cmidrule(lr){6-7}
Task Category  & base        & base + $\mathcal{L}_{\dot{a}}^{\text{det}}$         & base            & base + $\mathcal{L}_{\dot{a}}^{\text{sto}}$                  & base & base + $\mathcal{L}_{\dot{a}}^{\text{flow}}$             \\ 
\midrule
Hopper-e      & $21.5{\scriptstyle \pm 13.5}$ & $67.3{\scriptstyle \pm 14.3}({\color{mygreen}\uparrow\! 45.8})$ & $57.5{\scriptstyle \pm 15.2}$ & $93.1{\scriptstyle \pm 10.7}({\color{mygreen}\uparrow\! 35.6})$ & $72.2{\scriptstyle \pm 9.2}$  & $97.2{\scriptstyle \pm 12.3}({\color{mygreen}\uparrow\! 25.0})$ \\
\midrule
Halfcheetah-e & $2.4{\scriptstyle \pm 1.2}$   & $5.7{\scriptstyle \pm 1.5}({\color{mygreen}\uparrow\! 3.3})$    & $2.8{\scriptstyle \pm 1.6}$   & $16.2{\scriptstyle \pm 10.1}({\color{mygreen}\uparrow\! 13.4})$ & $11.4{\scriptstyle \pm 4.5}$  & $17.9{\scriptstyle \pm 6.9}({\color{mygreen}\uparrow\! 6.5})$  \\
\midrule
Walker2d-e    & $26.7{\scriptstyle \pm 12.2}$ & $91.1{\scriptstyle \pm 14.7}({\color{mygreen}\uparrow\! 64.4})$ & $46.4{\scriptstyle \pm 15.2}$ & $91.6{\scriptstyle \pm 9.5}({\color{mygreen}\uparrow\! 45.2})$  & $89.8{\scriptstyle \pm 14.6}$ & $91.5{\scriptstyle \pm 6.3}({\color{mygreen}\uparrow\! 1.7})$   \\
\bottomrule
\end{tabular}
}
\end{table}

\textbf{Evaluation results on D4RL under the full-dataset setting.}\quad
To further evaluate the effectiveness of our scheme, we additionally conduct experiments on the D4RL MuJoCo locomotion tasks under the standard full-size dataset setting. As reported in Table~\ref{tab:d4rl_res_full}, integrating our first-order action supervision yields consistent performance improvements across all four base algorithms, which further confirms the broad applicability of our scheme, showing it reliably stabilizes and improves offline policy learning across varying dataset sizes.

\textbf{Evaluation results of imitation learning methods on the D4RL 10k setting.}\quad
Within each policy paradigm, ``base'' denotes the standard behavior cloning baseline natively implemented under the corresponding framework, while ``base + $\mathcal{L}_{\dot{a}}^{(\cdot)}$'' denotes the same baseline equipped with our first-order action supervision loss. As shown in Table~\ref{tab:im_d4rl_res_10k}, our first-order supervision consistently improves all three behavior cloning baselines, demonstrating its standalone effectiveness even without any reward signal or RL-based optimization objective. In particular, on the deterministic \texttt{Walker2d-e} task, adding $\mathcal{L}_{\dot{a}}^{\text{det}}$ raises the standard BC score from $26.7$ to $91.1$, an absolute improvement of $64.4$. Such a substantial gain reveals that, under small-dataset settings, standard BC is highly susceptible to compounding errors along the rollout trajectory. By directly supervising higher-order action dynamics, our scheme provides the missing temporal regularization that aligns rollout actions with expert dynamics and effectively recovers the expert behavior. More detailed ablation studies that further validate the mechanism of first-order supervision are provided in the following subsections.
% \vspace{-0.1pt}

\textbf{Learning curves and policy rollout distributions of $||\dot{a}||$.}\quad
We additionally present learning curves and rollout distributions of $||\dot{a}||$ for several baseline policies and their first-order action supervised variants, covering both imitation learning and offline RL settings. These results extend the analysis in Figure~\ref{fig:compare_da} and Figure~\ref{fig:compare_td3bc} to additional algorithms and tasks, and they consistently show that adding our first-order supervision yields smoother learning curves and more concentrated $||\dot{a}||$ distributions, confirming that our scheme produces temporally consistent action evolution beyond the cases highlighted in the main text.
\begin{figure}[t]
    \centering
    \includegraphics[width=\textwidth]{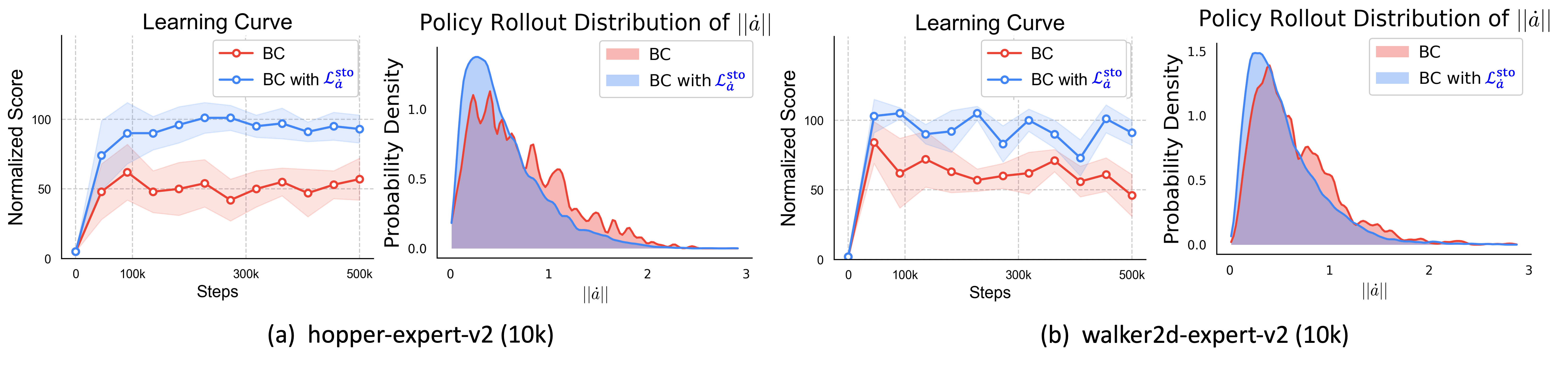}
    \caption{\small Comparison of vanilla stochastic BC and our $1^{\mathrm{st}}$-order action supervision on small-sample locomotion tasks. Scores are averaged over 5 random seeds. The policy rollout distributions of action variation $||\dot{a}||=||a'-a||$ ($a$ is joint torque) are aggregated from 10 trajectories.}
    \label{fig:compare_da_sto}
\end{figure}

\begin{figure}[t]
    \centering
    \includegraphics[width=\textwidth]{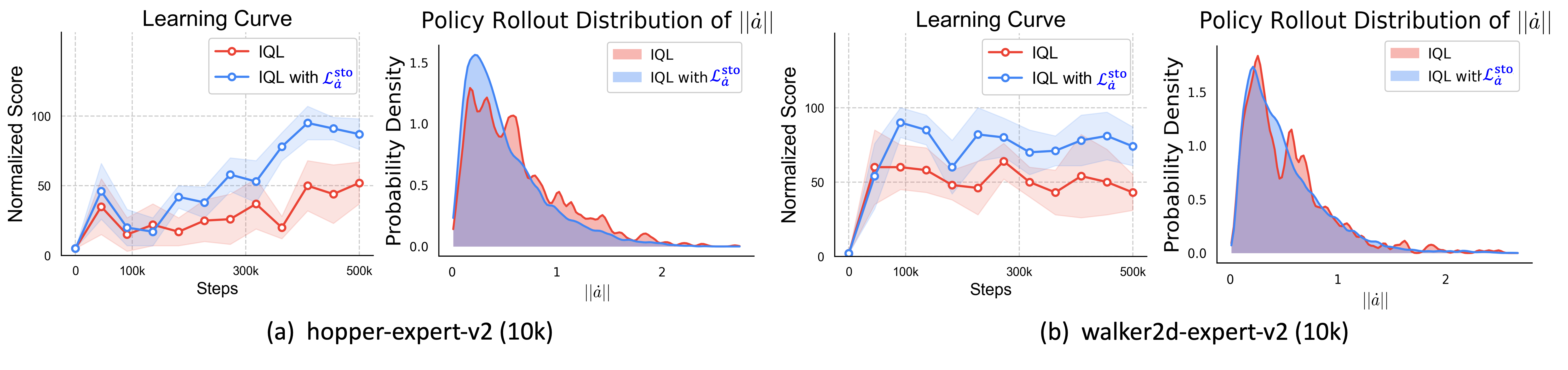}
    \caption{\small Comparison of vanilla IQL and our $1^{\mathrm{st}}$-order action supervision on small-sample locomotion tasks. Scores are averaged over 5 random seeds. The policy rollout distributions of action variation $||\dot{a}||=||a'-a||$ ($a$ is joint torque) are aggregated from 10 trajectories.}
    \label{fig:compare_da_iql}
\end{figure}

\textbf{Training time comparison.}\quad
Table~\ref{tab:train_time} compares the training time of each base algorithm against its variant equipped with our first-order action supervision on \texttt{hopper-medium-v2} (D4RL) and \texttt{cube-double} (OGBench). Across all four policy parameterizations, integrating our scheme only moderately increases the training time, indicating that our scheme is computationally efficient and practical for real-world use. All algorithms are trained on a workstation equipped with an Intel Xeon Platinum 8336C CPU @ 2.30 GHz and NVIDIA A100-SXM4-80GB GPUs, running Ubuntu 22.04.5 LTS 64-bit.
\begin{table}[H]
    \vspace{-15pt}
    \caption{\small Training time comparison between each base algorithm and its variant equipped with our first-order action supervision. All runs use the same number of training steps and are conducted on a single GPU.}
\label{tab:train_time}
\centering
\resizebox{\textwidth}{!}{%
\begin{tabular}{lcccccccc}
\toprule
& \multicolumn{2}{c}{Det. Policy} & \multicolumn{2}{c}{Sto. Policy (mean sup.)} & \multicolumn{2}{c}{Sto. Policy (standard)} & \multicolumn{2}{c}{Flow Policy} \\
\cmidrule(lr){2-3} \cmidrule(lr){4-5} \cmidrule(lr){6-7} \cmidrule(lr){8-9}
Task Category  & TD3+BC        & TD3+BC+ours          & ReBRAC            & ReBRAC+ours        & IQL            & IQL+ours                  & IFQL & IFQL+ours             \\ 
\midrule
Training time (D4RL) &28 min 30 s & 33 min 6 s & 25 min 8 s & 30 min 34 s & 34 min 44 s & 39 min 56 s & 30 min 49 s & 35 min 52 s\\
Training time (OGBench) &28 min 23 s & 38 min 56 s & 25 min 11 s & 31 min 54 s & 35 min 7 s & 42 min 5 s & 32 min 0 s & 36 min 51 s\\
\bottomrule
\end{tabular}
}
\end{table}

\textbf{Generalizing to online RL.}\quad
To evaluate the versatility of our approach beyond the offline setting, we integrate both the first-order action supervision and the continuous-time HJB value regularization into the standard SAC~\citep{haarnoja2018soft} algorithm, and evaluate on three MuJoCo benchmarks over 1M training steps. We keep SAC's original hyperparameters intact, and set $\lambda_{\dot{a}} = \lambda_Q = 0.1$ for Hopper-v2 and Walker2d-v2, and $\lambda_{\dot{a}} = 0.01, \lambda_Q = 0.1$ for HalfCheetah-v2.

As reported in Table~\ref{tab:online_rl}, incorporating our first-order supervision leads to noticeably faster convergence, which is particularly evident at 300k steps, alongside reduced performance variance across seeds. These results confirm that our lightweight, plug-and-play module generalizes effectively to the online RL paradigm, improving sample efficiency and training stability without modifying the base algorithm.

\begin{table}[t]
    \caption{\small Performance snapshots of SAC and SAC+Ours on online MuJoCo benchmarks over 1M training steps. Results are averaged over 5 random seeds, with standard deviations reported after the $\pm$ sign. \textbf{Bold} indicates the better result at each snapshot.}
    \label{tab:online_rl}
    \begin{center}
        \resizebox{\textwidth}{!}{
            \begin{tabular}{llccccc}
            \toprule
            Environment & Method & 100k & 300k & 500k & 700k & 1M \\
            \midrule
            Hopper-v2 & SAC & $1767.1{\scriptstyle \pm 1005.9}$ & $2399.1{\scriptstyle \pm 1391.8}$ & $3076.8{\scriptstyle \pm 405.9}$ & $\bm{3389.6{\scriptstyle \pm 162.2}}$ & $3530.4{\scriptstyle \pm 32.6}$ \\
            & SAC + Ours & $\bm{2375.8{\scriptstyle \pm 1036.4}}$ & $\bm{3273.9{\scriptstyle \pm 111.6}}$ & $\bm{3226.4{\scriptstyle \pm 101.1}}$ & $3359.2{\scriptstyle \pm 116.4}$ & $\bm{3531.8{\scriptstyle \pm 33.2}}$ \\
            \midrule
            HalfCheetah-v2 & SAC & $5659.2{\scriptstyle \pm 676.4}$ & $8564.9{\scriptstyle \pm 1699.7}$ & $9750.7{\scriptstyle \pm 1228.6}$ & $10805.1{\scriptstyle \pm 804.7}$ & $11557.4{\scriptstyle \pm 493.4}$ \\
            & SAC + Ours & $\bm{5790.9{\scriptstyle \pm 553.5}}$ & $\bm{8614.4{\scriptstyle \pm 1423.9}}$ & $\bm{9954.8{\scriptstyle \pm 1126.3}}$ & $\bm{10893.5{\scriptstyle \pm 612.1}}$ & $\bm{11858.2{\scriptstyle \pm 414.1}}$ \\
            \midrule
            Walker2d-v2 & SAC & $702.9{\scriptstyle \pm 343.5}$ & $2413.7{\scriptstyle \pm 1150.3}$ & $2989.8{\scriptstyle \pm 1820.9}$ & $4291.2{\scriptstyle \pm 950.5}$ & $4768.9{\scriptstyle \pm 393.1}$ \\
            & SAC + Ours & $\bm{869.9{\scriptstyle \pm 310.3}}$ & $\bm{3300.3{\scriptstyle \pm 879.6}}$ & $\bm{4141.4{\scriptstyle \pm 572.9}}$ & $\bm{4805.8{\scriptstyle \pm 184.5}}$ & $\bm{4991.1{\scriptstyle \pm 264.5}}$ \\
            \bottomrule
            \end{tabular}
        }
    \end{center}
\end{table}

\subsection{Additional Ablation Results}
\label{Additional Ablation Results}
To verify that the benefits of our scheme are not specific to a single
algorithmic instantiation, we conduct ablation studies on TD3+BC
(deterministic policy) and IQL (stochastic policy), evaluated on the
 D4RL \texttt{hopper-expert-v2} and \texttt{walker2d-expert-v2} tasks under the 10k dataset setting.
We sweep one of the two weights over $\{0, 0.5, 1\}$ while strictly setting
the other to zero, so that the contribution of each component can be
attributed independently. All results are averaged over 5 random seeds.
\begin{figure}[t]
    \centering
    % 第一个图
    \begin{minipage}{0.48\textwidth}
        \centering
        \includegraphics[width=\linewidth]{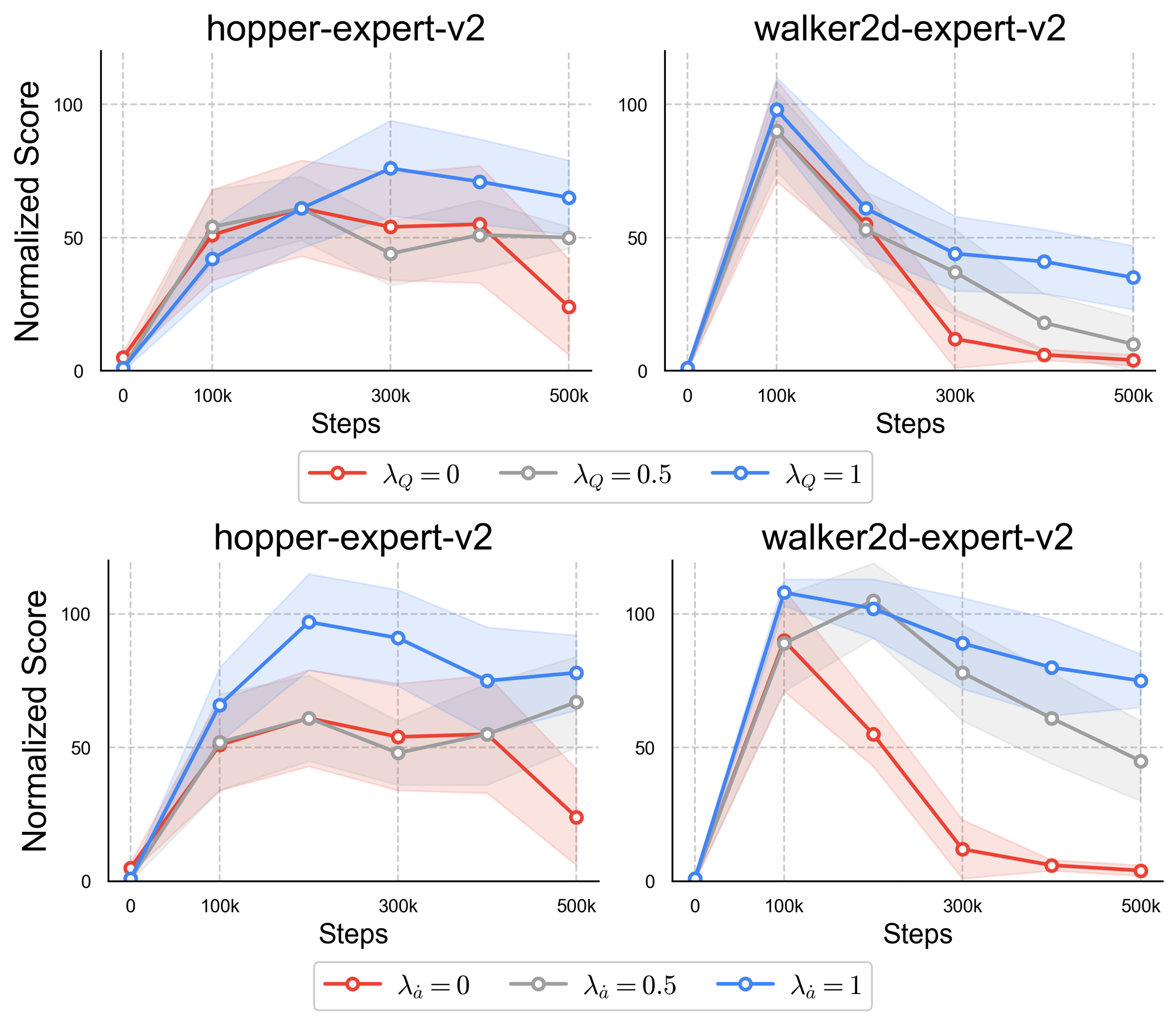}
        \caption{\small Ablation study on value regularization weight $\lambda_Q$ and $1^\text{st}$-order action supervision weight $\lambda_{\dot{a}}$ of TD3+BC. }
        \label{fig:td3bc_lambda_abla}
    \end{minipage}
    \hspace{0.05cm}
    % 第二个图
    \begin{minipage}{0.48\textwidth}
        \centering
        \includegraphics[width=\linewidth]{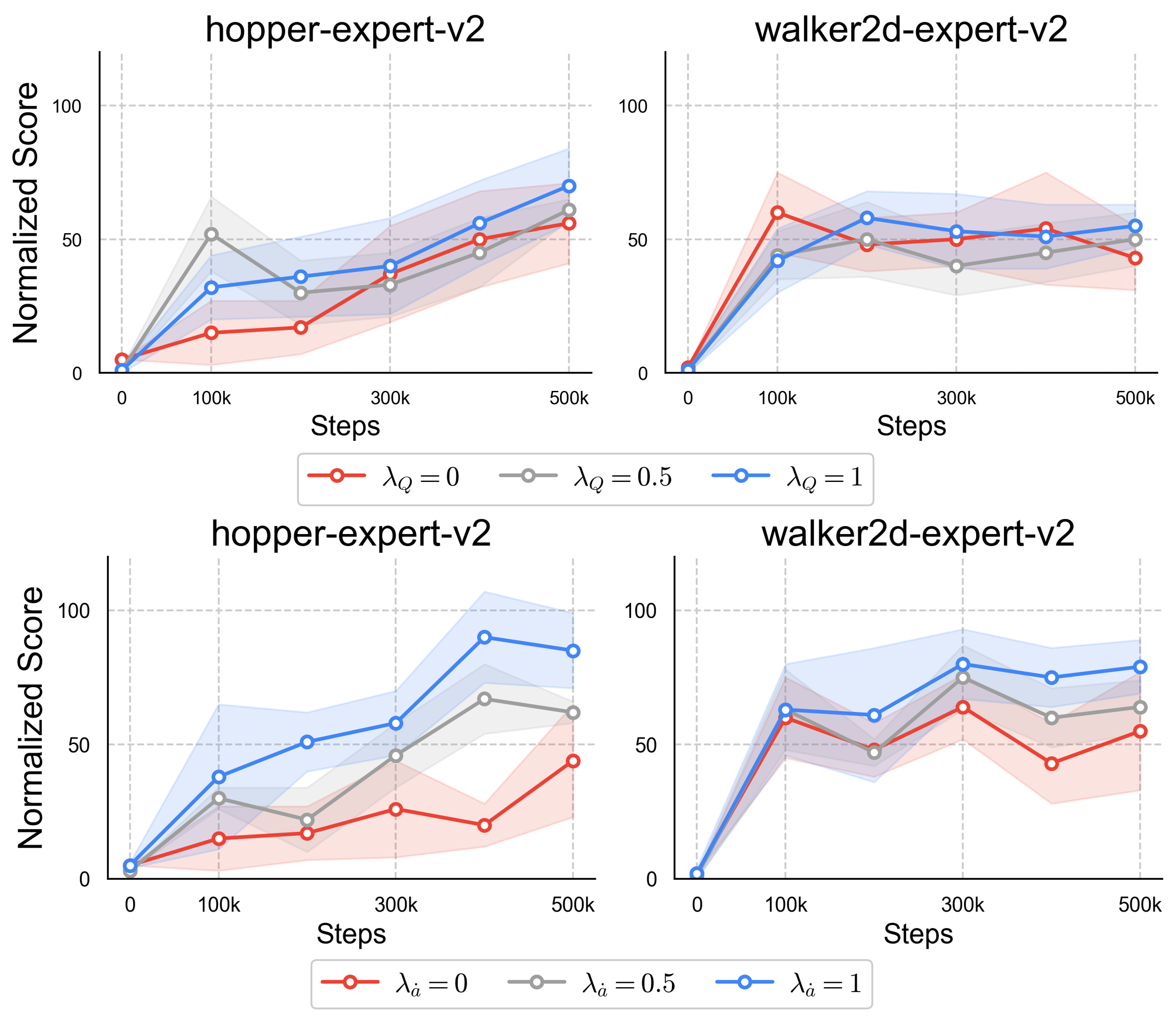}
        \caption{\small Ablation study on value regularization weight $\lambda_Q$ and $1^\text{st}$-order action supervision weight $\lambda_{\dot{a}}$ of IQL}
        \label{fig:iql_lambda_abla}
    \end{minipage}
\end{figure}

\begin{table}[t]
    \caption{\small Performance on OGBench environments. Results are reported as mean $\pm$ standard deviation. \textbf{Bold} indicates the highest mean in each environment.}
    \label{tab:ogbench}
    \begin{center}
        \resizebox{\textwidth}{!}{
            \begin{tabular}{lccccc}
                \toprule
                Environment (OGBENCH) & IFQL & +Ours & +AK ($\lambda_a=0.5$) & +AK ($\lambda_a=0.1$) & +AK ($\lambda_a=1$) \\
                \midrule
                humanoidmaze-medium-navigate-singletask-task1-v0
                & $69{\scriptstyle \pm 19}$ & $71{\scriptstyle \pm 16}$
                & $\bm{73.2{\scriptstyle \pm 7.0}}$ & $51.2{\scriptstyle \pm 32.4}$ & $67.6{\scriptstyle \pm 7.1}$ \\
                humanoidmaze-medium-navigate-singletask-task4-v0
                & $\bm{1{\scriptstyle \pm 1}}$ & $\bm{1{\scriptstyle \pm 1}}$
                & $0.0{\scriptstyle \pm 0.0}$ & $0.4{\scriptstyle \pm 0.8}$ & $0.0{\scriptstyle \pm 0.0}$ \\
                \midrule
                antsoccer-arena-navigate-singletask-task2-v0
                & $75{\scriptstyle \pm 3}$ & $\bm{83{\scriptstyle \pm 12}}$
                & $66.8{\scriptstyle \pm 15.3}$ & $40.4{\scriptstyle \pm 33.6}$ & $35.2{\scriptstyle \pm 20.4}$ \\
                antsoccer-arena-navigate-singletask-task4-v0
                & $16{\scriptstyle \pm 9}$ & $16{\scriptstyle \pm 8}$
                & $13.6{\scriptstyle \pm 9.6}$ & $\bm{20.4{\scriptstyle \pm 20.8}}$ & $6.4{\scriptstyle \pm 7.3}$ \\
                \midrule
                cube-single-play-singletask-task2-v0
                & $73{\scriptstyle \pm 3}$ & $\bm{100{\scriptstyle \pm 1}}$
                & $90.0{\scriptstyle \pm 7.0}$ & $27.2{\scriptstyle \pm 33.5}$ & $94.0{\scriptstyle \pm 2.8}$ \\
                cube-single-play-singletask-task4-v0
                & $79{\scriptstyle \pm 6}$ & $88{\scriptstyle \pm 16}$
                & $\bm{93.6{\scriptstyle \pm 2.3}}$ & $26.8{\scriptstyle \pm 34.2}$ & $92.0{\scriptstyle \pm 4.0}$ \\
                \midrule
                cube-double-play-singletask-task1-v0
                & $35{\scriptstyle \pm 9}$ & $\bm{76{\scriptstyle \pm 8}}$
                & $23.6{\scriptstyle \pm 7.5}$ & $29.6{\scriptstyle \pm 5.6}$ & $25.6{\scriptstyle \pm 10.8}$ \\
                cube-double-play-singletask-task4-v0
                & $1{\scriptstyle \pm 1}$ & $\bm{12{\scriptstyle \pm 10}}$
                & $1.2{\scriptstyle \pm 1.0}$ & $1.2{\scriptstyle \pm 1.0}$ & $1.2{\scriptstyle \pm 1.0}$ \\
                \midrule
                puzzle-4x4-play-singletask-task2-v0
                & $4{\scriptstyle \pm 4}$ & $\bm{12{\scriptstyle \pm 9}}$
                & $2.8{\scriptstyle \pm 5.6}$ & $0.0{\scriptstyle \pm 0.0}$ & $8.4{\scriptstyle \pm 4.3}$ \\
                puzzle-4x4-play-singletask-task3-v0
                & $50{\scriptstyle \pm 14}$ & $\bm{84{\scriptstyle \pm 16}}$
                & $5.6{\scriptstyle \pm 10.2}$ & $0.0{\scriptstyle \pm 0.0}$ & $37.2{\scriptstyle \pm 7.8}$ \\
                \bottomrule
            \end{tabular}
        }
    \end{center}
\end{table}
\begin{table}[t]
    \caption{\small Computational cost and GPU memory usage of Ours and AK. Additional overhead is measured relative to Ours.}
    \label{tab:computational_overhead}
    \begin{center}
        \resizebox{\textwidth}{!}{
            \begin{tabular}{lccc}
                \toprule
                Metric & Ours & AK & Additional overhead \\
                \midrule
                Computational cost (per update)
                & 2.498 GFLOPs & 4.978 GFLOPs
                & \textbf{+2.480 GFLOPs (+99.3\%)} \\
                GPU memory
                & 1.141 GiB & 1.266 GiB
                & \textbf{+0.125 GiB (+11.0\%)} \\
                \bottomrule
            \end{tabular}
        }
    \end{center}
\end{table}

As shown in Figures~\ref{fig:td3bc_lambda_abla} and~\ref{fig:iql_lambda_abla}, both algorithms exhibit the same trend:
increasing either the first-order action supervision weight $\lambda_{\dot{a}}$
or the value regularization weight $\lambda_Q$ from 0 to 1 yields steady
performance gains. This
confirms that the policy-side and value-side first-order regularizers
are robustly beneficial across both deterministic and stochastic policy
classes.

To evaluate the action-dependent term omitted from
Eq.~(\ref{eq:flow_policy}), we compare Ours with an ablation denoted as AK. AK adds
$\frac{\partial v(k,s,a^k)}{\partial a^k}\dot a^k$
from Eq.~(\ref{flow_da}) to the state-directional term, 
uses the same objective as Eq.~(\ref{eq:flow_policy}),
replacing $\mathcal{L}_{\dot a}^{\mathrm{flow}}$ with
\begin{equation}
\label{eq:ak_loss}
\mathcal{L}_{\dot a}^{\mathrm{AK}}
=
\left\|
\frac{\partial v(k,s,a^k)}{\partial s}\dot s
+
\frac{\partial v(k,s,a^k)}{\partial a^k}\dot a^k
-
\dot a^1
\right\|.
\end{equation}
Here $\dot a^k=k\dot a^1$ on the linear interpolation path.

Overall, Table~\ref{tab:ogbench} shows that Ours performs more
consistently across OGBench environments. It matches or improves
upon IFQL on every task and generally outperforms AK across the
tested values of $\lambda_a$. Although AK benefits a few tasks,
it produces substantial drops on others, particularly
\texttt{cube-double} and \texttt{puzzle-4x4}. Thus, including
the action-dependent term does not yield a reliable performance
benefit.

Table~\ref{tab:computational_overhead} further shows that AK
nearly doubles computation per update (+99.3\%) and increases
GPU memory usage by 11.0\% relative to Ours. Together, these
results support our empirical choice of the state-directional
approximation.

\section{Derivation of the Higher-Order Supervision}
This section provides the derivations of the higher-order policy as demonstrated in section~\ref{sec:policy_class}. Specifically, Appendix~\ref{app:sto_policy_derivation} derives the action differential for stochastic Gaussian policies, while Appendix~\ref{app:flow_policy_derivation} derives the time derivative of the generated action for flow-based policies.

\subsection{First-Order Stochastic Gaussian Policies Derivation in Eq. (\ref{stochastic_da})}
\label{app:sto_policy_derivation}

Consider the Gaussian policy with diagonal covariance
\begin{equation}
\pi(a|s)
=
\mathcal{N}\!\left(
a;\mu(s),\operatorname{diag}(\sigma^2(s))
\right),
\end{equation}
where $a,\mu(s),\sigma(s)\in\mathbb{R}^d$, and $\sigma^2(s)$ denotes the
element-wise variance. We use $\sigma(s)$ to denote the element-wise standard
deviation, i.e., the element-wise square root of the variance output.

We represent the stochastic component of the sampled action by
\begin{equation}
\label{sto_at}
a
=
\mu(s)+\sigma(s)\odot X_t,
\end{equation}
where $\odot$ denotes element-wise multiplication, and $X_t\in\mathbb{R}^d$ is
assumed to follow the stationary vector OU process \citep{uhlenbeck1930theory}
\begin{equation}
\label{OU}
\mathrm{d}X_t
=
-\lambda X_t\,\mathrm{d}t
+
\sqrt{2\lambda}\,\mathrm{d}W_t,
\qquad
X_0\sim\mathcal{N}(0,I).
\end{equation}
Here $W_t$ is a $d$-dimensional standard Wiener process. By stationarity,
$X_t\sim\mathcal{N}(0,I)$ for all $t$. Under the diagonal covariance assumption,
the stochastic perturbation is applied independently across action dimensions.
Therefore,
\[
a=\mu(s)+\sigma(s)\odot X_t,
\qquad
X_t\sim\mathcal{N}(0,I)
\quad\Longrightarrow\quad
a\sim
\mathcal{N}\!\left(
\mu(s),\operatorname{diag}(\sigma^2(s))
\right).
\]

We condition on $(s,\dot{s})$ and treat $s=s(t)$ i.e.,
$\mathrm{d}s = \dot{s}\,\mathrm{d}t$
as a differentiable trajectory.
For notational simplicity, we omit the explicit dependence of $s$ on $t$.
Define
\[
F(t,X_t)
=
\mu(s)+\sigma(s)\odot X_t,
\qquad
a=F(t,X_t).
\]
For each action dimension $i$, define
\begin{equation}
\label{F_i}
a_i=F_i(t,X_t)
=
\mu_i(s)+\sigma_i(s)X_{t,i}.
\end{equation}
Since $X_t$ is a vector Itô process, applying the multi-dimensional Itô formula
\citep{oksendal2013stochastic} to $F_i(t,X_t)$ and substituting the OU
dynamics of $X_t$ gives,
\begin{align}
\mathrm{d}a_i
&=
\Bigg[
\frac{\partial F_i}{\partial t}
+
\nabla_X F_i(t,X_t)^\top
\left(
-\lambda X_t
\right)
\nonumber
+
\frac{1}{2}
\operatorname{Tr}
\left(
(\sqrt{2\lambda}I)^\top
H_XF_i(t,X_t)
(\sqrt{2\lambda}I)
\right)
\Bigg]\mathrm{d}t
\nonumber\\
&\qquad
+
\nabla_XF_i(t,X_t)^\top
\sqrt{2\lambda}I\,\mathrm{d}W_t
\nonumber\\
&=
\left[
\frac{\partial F_i}{\partial t}
-
\lambda
\nabla_XF_i(t,X_t)^\top X_t
+
\lambda
\operatorname{Tr}
\left(
H_XF_i(t,X_t)
\right)
\right]\mathrm{d}t
\nonumber\\
&\qquad
+
\sqrt{2\lambda}\,
\nabla_XF_i(t,X_t)^\top
\mathrm{d}W_t .
\end{align}
From Eq.~(\ref{F_i}), $F_i$ is linear with respect to $X_t$.
\begin{equation}
    H_XF_i(t,X_t)=0.
\end{equation}
Moreover,
\begin{equation}
\frac{\partial F_i}{\partial t}
=
\dot{\mu}_i(s)+\dot{\sigma}_i(s)X_{t,i},
\qquad
\nabla_XF_i(t,X_t)
=
\sigma_i(s)e_i,
\end{equation}
where $e_i$ is the $i$-th standard basis vector. Substituting these identities
into the previous expression gives
\begin{equation}
\mathrm{d}a_i
=
\left[
\dot{\mu}_i(s)
+
\left(
\dot{\sigma}_i(s)-\lambda\sigma_i(s)
\right)X_{t,i}
\right]\mathrm{d}t
+
\sqrt{2\lambda}\,\sigma_i(s)\mathrm{d}W_{t,i}.
\end{equation}
Stacking all action dimensions yields
\begin{equation}
\mathrm{d}a
=
\left[
\dot{\mu}(s)
+
\left(
\dot{\sigma}(s)-\lambda\sigma(s)
\right)\odot X_t
\right]\mathrm{d}t
+
\sqrt{2\lambda}\,\sigma(s)\odot \mathrm{d}W_t .
\end{equation}

Using Eq.~(\ref{sto_at}), we obtain
\begin{equation}
X_t
=
\left(a-\mu(s)\right)\oslash\sigma(s),
\end{equation}
where $\oslash$ denotes element-wise division. Hence
\begin{equation}
\label{stochastic_da_app}
\mathrm{d}a
=
\left[
\dot{\mu}(s)
+
\left(
\dot{\sigma}(s)-\lambda\sigma(s)
\right)\odot
\left(
\left(a-\mu(s)\right)\oslash\sigma(s)
\right)
\right]\mathrm{d}t
+
\sqrt{2\lambda}\,\sigma(s)\odot \mathrm{d}W_t .
\end{equation}

Since the infinitesimal Wiener increment satisfies
$\mathrm{d}W_t \sim \mathcal{N}(0,I\mathrm{d}t)$,
Eq.~(\ref{stochastic_da_app}) implies that the infinitesimal action increment
conditioned on $(s,a,\dot{s})$ follows
\begin{equation}
\mathrm{d}a
\sim
\mathcal{N}\!\left(
\left[
\dot{\mu}(s)
+
\left(
\dot{\sigma}(s)-\lambda\sigma(s)
\right)
\odot
\left(
\left(a-\mu(s)\right)\oslash\sigma(s)
\right)
\right]\mathrm{d}t,\,
2\lambda\operatorname{diag}(\sigma^2(s))\mathrm{d}t
\right).
\end{equation}

As mentioned in Section~\ref{notations}, in our discrete-time implementation, we use the finite-difference approximation
$\dot{a}\approx (a'-a)/\Delta t$ with $\Delta t=1$.
Under this unit-step finite-difference approximation, the induced first-order policy distribution over the action derivative is
\begin{equation}
\dot{\pi}(\dot a|s,a,\dot{s})
=
\mathcal{N}\!\left(
\dot a;\,
\dot{\mu}(s)
+
\left(
\dot{\sigma}(s)-\lambda\sigma(s)
\right)
\odot
\left(
\left(a-\mu(s)\right)\oslash\sigma(s)
\right),
\,
2\lambda\operatorname{diag}(\sigma^2(s))
\right).
\end{equation}

\subsection{First-Order Flow Policies Derivation in Eq. (\ref{flow_da})}
\label{app:flow_policy_derivation}
% For an ideal flow policy, 
Given the state $s$ and initial noise sample $a^0$, the generated action is obtained by solving
\begin{equation}
\frac{d a^k}{d k}
=
v(k,s,a^k),
\qquad
a^{k=0}=a^0,
\end{equation}
and taking $a^1=\pi(s)$,
\begin{equation}
\label{app:a1}
a^1
=
a^0+\int_0^1 v(k,s,a^k)\,dk .
\end{equation}

Assuming sufficient smoothness, differentiating $a^1$ with respect to physical time $t$ and applying the Leibniz integral rule gives
\begin{align}
\dot{a}^1
&=
\dot a^0 +
\frac{d}{dt}
\int_0^1 v(k,s,a^k)\,dk \notag \\
&=
\dot a^0 +
\int_0^1
\frac{d}{dt}
v(k,s,a^k)\,dk .
\label{dot_a_1}
\end{align}
Since $v(k,s,a^k)$ depends on $t$ through both $s$ and $a^k$, the chain rule gives
\begin{equation}
\label{chain}
\frac{d}{dt}v(k,s,a^k)
=
\frac{\partial v(k,s,a^k)}{\partial s}\dot{s}
+
\frac{\partial v(k,s,a^k)}{\partial a^k}\dot a^k.
\end{equation}
Conditioning on the sampled noise $a^0$ when taking the
physical-time derivative gives $\dot a^0=0$. Combining
Eqs.~(\ref{dot_a_1}) and~(\ref{chain}) yields
\begin{equation}
\label{flow_da_app}
\dot a^1
=
\int_0^1 \dot v(k,s,a^k)\,dk
=
\int_0^1
\left(
\frac{\partial v(k,s,a^k)}{\partial s}\dot s
+
\frac{\partial v(k,s,a^k)}{\partial a^k}\dot a^k
\right)dk.
\end{equation}

We consider the linear interpolation path used in flow matching,
\begin{equation}
a^k=(1-k)a^0+ka^1,\qquad k\in[0,1].
\end{equation}
Taking its physical-time derivative while holding $a^0$ fixed gives
\begin{equation}
\dot a^k=(1-k)\dot a^0+k\dot a^1=k\dot a^1.
\end{equation}
Equation~(\ref{flow_da_app}) therefore takes the following form
along this path,
\begin{equation}
\dot a^1
=
\int_0^1
\left(
\frac{\partial v(k,s,a^k)}{\partial s}\dot s
+
k\frac{\partial v(k,s,a^k)}{\partial a^k}\dot a^1
\right)dk.
\end{equation}
As in the zeroth-order action generation process in
Eq.~(\ref{app:a1}), the first-order action is obtained by
integrating a vector field over flow time $k$. For the linear
path, the sample-wise flow-matching target is
$\frac{d a^k}{dk}=a^1-a^0$. Taking its physical-time
derivative while holding $a^0$ fixed gives
$\frac{d}{dt}(a^1-a^0)=\dot a^1$. Thus,
$\dot a^1=a'-a^1$ provides a pointwise target for
$\dot v(k,s,a^k)$. The shared integral structure motivates
supervising this first-order field in a flow-matching manner.

Based on the computational cost and performance results in
Appendix~\ref{Additional Ablation Results}, we approximate
$\dot v(k,s,a^k)$ by
$\frac{\partial v(k,s,a^k)}{\partial s}\dot s$.
Accordingly, Eq.~(\ref{eq:flow_policy})
supervises this approximation at sampled values of $k$;
the integral describes action generation rather than the
training loss.

\section{Theoretical Analysis}
\label{app:theory}
We first introduce the formal notation, assumptions, and auxiliary lemmas used in the analysis. Then we provide the proofs of Theorems~\ref{theorem2} and~\ref{theorem1}.
% We follow the notation introduced in Section~\ref{theoretical analysis}. 
% This appendix briefly recaps the notation used in the analysis, states the regularity assumptions, and provides the proofs of Theorems~\ref{theorem2} and~\ref{theorem1}.

\subsection{Notation}
\label{app:notation}
We briefly recap the formal notation used in Section~\ref{theoretical analysis}, with the time dependence made explicit for the analysis below.
Let $\pi$ and $\pi^*$ denote the learned and expert policies, and let $s(t)$ and $s^*(t)$ denote their corresponding trajectories. 
The closed-loop dynamics are $\dot{s}(t)=f(s(t),\pi(s(t)))$ and $\dot{s}^*(t)=f(s^*(t),\pi^*(s^*(t)))$. 
The policy error and state deviation are denoted by $e(s)=\pi(s)-\pi^*(s)$ and $\delta(t)=s(t)-s^*(t)$, respectively. 
For each time $t$, define $T(t)=\{s^*(t)+\tau(s(t)-s^*(t)):\tau\in[0,1]\}$.

\subsection{Assumptions}
\label{assumptions}
\begin{assumption}[Domain and regularity]
\label{ass:domain}
There exists a convex compact set $\Omega \subset \mathbb{R}^d$ such that $s(t),s^*(t)\in\Omega$ for all $t$ under consideration. 
The policies $\pi$ and $\pi^*$ are continuously differentiable on $\Omega$, and the policy error $e$ admits the Jacobian $\nabla_s e$ on $\Omega$.
\end{assumption}

\begin{assumption}[Lipschitz dynamics]\label{ass:lipschitz_dynamics}
For any $s_1, s_2 \in \Omega$ and any $a \in \mathbb{R}^m$, the dynamics $f$ is Lipschitz continuous in the state variable with constant $L_s$, i.e.,
\[
\|f(s_1, a) - f(s_2, a)\| \le L_s \|s_1 - s_2\|.
\]
For any $s \in \Omega$ and any $a_1, a_2 \in \mathbb{R}^m$, the dynamics $f$ is Lipschitz continuous in the action variable with constant $L_a$, i.e.,
\[
\|f(s, a_1) - f(s, a_2)\| \le L_a \|a_1 - a_2\|.
\]
\end{assumption}

\begin{assumption}[Lipschitz expert policy]\label{ass:lipschitz_expert}
The expert policy $\pi^*$ is Lipschitz continuous with constant $L_{\pi^*}$ on $\Omega$, i.e., for any $s_1, s_2 \in \Omega$,
\[
\|\pi^*(s_1) - \pi^*(s_2)\| \le L_{\pi^*} \|s_1 - s_2\|.
\]
\end{assumption}

\begin{assumption}[Bounded $\nabla_s e$]
\label{ass:bounded_jacobian}
For all $t$ under consideration, the Jacobian of the policy error is bounded on $T(t)$ by $L_e$. That is, for any $z\in T(t)$,
\[
\left\|\nabla_s e(z)\right\|\le L_e.
\]
\end{assumption}

\subsection{Auxiliary Lemmas}
\begin{lemma}
\label{lem:policy_error_bound}
Under Assumption~\ref{ass:domain}, for any \(s_1,s_2\in\Omega\), letting
\(\Delta=s_1-s_2\) and \(\gamma(\tau)=s_2+\tau\Delta\) for \(\tau\in[0,1]\), we have
\[
e(s_1)-e(s_2)
=
\int_0^1 
\nabla_s e(\gamma(\tau))
\,\Delta\,d\tau.
\]
Consequently, for any time \(t\), along the segment \(T(t)\), we have
\[
\|e(s(t))\|
\le
\|e(s^*(t))\|
+
\sup_{z \in T(t)}
\left\|
\nabla_s e(z)
\right\|
\, \|s(t)-s^*(t)\|.
\]
\end{lemma}

\begin{proof}
Fix \(s_1,s_2\in\Omega\), and define
\(\Delta=s_1-s_2\) and \(\gamma(\tau)=s_2+\tau\Delta\).
By Assumption~\ref{ass:domain}, \(\gamma(\tau)\in\Omega\) for all \(\tau\in[0,1]\).
Let \(g(\tau)=e(\gamma(\tau))\). Since \(e\) is continuously differentiable by
Assumption~\ref{ass:domain}, the chain rule gives
\[
g'(\tau)
=
\nabla_s e(\gamma(\tau))
\,\Delta.
\]
By the fundamental theorem of calculus,
\[
e(s_1)-e(s_2)=g(1)-g(0)=\int_0^1 g'(\tau)\,d\tau
=
\int_0^1 
\nabla_s e(\gamma(\tau))
\,\Delta\,d\tau.
\]

Now take \(s_1=s(t)\) and \(s_2=s^*(t)\). By the definitions of \(\delta(t)\) and \(T(t)\),
we have \(\Delta=\delta(t)\) and \(\gamma(\tau)\in T(t)\). Therefore,
\[
e(s(t))-e(s^*(t))
=
\int_0^1 
\nabla_s e(\gamma(\tau))
\,\delta(t)\,d\tau.
\]
Taking norms yields
\[
\begin{aligned}
\|e(s(t))\|
&\le
\|e(s^*(t))\|
+
\int_0^1 
\left\|
\nabla_s e(\gamma(\tau))
\,\delta(t)
\right\|\,d\tau \\
&\le
\|e(s^*(t))\|
+
\int_0^1 
\left\|
\nabla_s e(\gamma(\tau))
\right\|
\,\|\delta(t)\|\,d\tau \\
&\le
\|e(s^*(t))\|
+
\sup_{z\in T(t)}
\left\|
\nabla_s e(z)
\right\|
\,\|\delta(t)\|.
\end{aligned}
\]
Using \(\delta(t)=s(t)-s^*(t)\), we obtain
\[
\|e(s(t))\|
\le
\|e(s^*(t))\|
+
\sup_{z\in T(t)}
\left\|
\nabla_s e(z)
\right\|
\,\|s(t)-s^*(t)\|.
\]
This completes the proof.
\end{proof}

\begin{lemma}[Gronwall's lemma]\label{lem:gronwall} \citep{emmrich_1999}
Let $T\in(0,\infty)$. Assume $u$ is absolutely continuous on $[0,T)$ and satisfies the differential inequality
\[
\dot u(t) \le a(t)\,u(t) + b(t), \qquad \text{for a.e. } t\in(0,T),
\]
for some $a,b\in L^1(0,T)$. Then $u$ satisfies the pointwise estimate
\[
u(t) \le e^{A(t)}u(0) + \int_0^t b(s)\,e^{A(t)-A(s)}\,ds,
\qquad \forall\, t\in[0,T),
\]
where we define
\[
A(t) = \int_0^t a(s)\,ds.
\]
This is the differential form of Gronwall's inequality, see Proposition~2.2 of \citep{emmrich_1999} for a proof.
\end{lemma}

\subsection{Proof of Theorem~\ref{theorem2}}
\label{app:theorem2_proof}
\begin{theoremonefull}
\label{theorem2_full}
Under Assumption~\ref{ass:domain}, the expert-trajectory policy error $e(s^*(t))$ satisfies:
\[
\|e(s^*(t))\|
\le
\|e(s^*(0))\|
+
t \,
\max_{0 \le \tau \le t}
\bigg\|
\nabla_s\pi(s^*(\tau))\dot{s}^*(\tau)
-
\dot{\pi}^*(s^*(\tau))
\bigg\|.
\]
\end{theoremonefull}

\begin{proof}
Recall the definition of the expert-trajectory policy error:
\[
e(s^*(t)) = \pi(s^*(t)) - \pi^*(s^*(t)).
\]
Define the expert action derivative along the expert trajectory as
\[
\dot{\pi}^*(s^*(t))
=
\frac{d}{dt} \pi^*(s^*(t))
=
\nabla_s \pi^*(s^*(t))\,\dot{s}^*(t).
\]
By Assumption~\ref{ass:domain} and the chain rule, the time derivative of the expert-trajectory policy error is
\[
\frac{d}{dt}e(s^*(t))
=
\nabla_s \pi (s^*(t))\,\dot{s}^*(t)
-
\dot \pi^*(s^*(t)).
\]

Integrating over time yields:
\[
e(s^*(t))
=
e(s^*(0))
+
\int_0^t
\bigg(
\nabla_s \pi (s^*(\tau))\,\dot{s}^*(\tau)
-
\dot{\pi}^*(s^*(\tau))
\bigg)
\, d\tau .
\]
Taking norms:
\[
\|e(s^*(t))\|
\le
\|e(s^*(0))\|
+
\int_0^t
\bigg\|
\nabla_s \pi (s^*(\tau))\,\dot{s}^*(\tau)
-
\dot{\pi}^*(s^*(\tau))
\bigg\|
\, d\tau .
\]

Therefore,
\[
\|e(s^*(t))\|
\le
\|e(s^*(0))\|
+
t \,
\max_{0 \le \tau \le t}
\bigg\|
\nabla_s \pi (s^*(\tau))\,\dot{s}^*(\tau)
-
\dot{\pi}^*(s^*(\tau))
\bigg\|.
\]
This completes the proof.
\end{proof}

\subsection{Proof of Theorem~\ref{theorem1}}
\label{app:theorem1_proof}

\begin{theoremtwofull}
\label{theorem2_full}
Under Assumptions~\ref{ass:domain}--\ref{ass:bounded_jacobian}, for $\delta(0)=0$, the state deviation $\delta(t)=s(t)-s^*(t)$ satisfies
\[
\|\delta(t)\|
\le
\frac{L_a}{\beta}\big(e^{\beta t}-1\big)
\max_{0\le \tau\le t}\|e(s^*(\tau))\|, \quad\text{where}\quad \beta=L_s+L_aL_{\pi^*}+L_aL_e.
\]
\end{theoremtwofull}

\begin{proof}
From Lemma~\ref{lem:policy_error_bound} and Assumption~\ref{ass:bounded_jacobian}, we have:
\begin{equation}
\label{eq:policy_error_bound_app}
\|e(s(t))\|
\le
\|e(s^*(t))\|
+
L_e \|s(t)-s^*(t)\|.
\end{equation}
The time derivative of the state deviation is:
\[
\dot{\delta}(t)
=
f(s(t),\pi(s(t))) - f(s^*(t),\pi^*(s^*(t))).
\]
By Assumptions~\ref{ass:lipschitz_dynamics} and~\ref{ass:lipschitz_expert}, we have
\[
\begin{aligned}
\|\dot{\delta}(t)\|
&=
\|f(s(t),\pi(s(t))) - f(s^*(t),\pi^*(s^*(t)))\| \\
&\le
\|f(s(t),\pi(s(t))) - f(s(t),\pi^*(s(t)))\|
+
\|f(s(t),\pi^*(s(t))) - f(s^*(t),\pi^*(s^*(t)))\| \\
&\le
L_a\|\pi(s(t))-\pi^*(s(t))\|
+
\|f(s(t),\pi^*(s(t))) - f(s^*(t),\pi^*(s(t)))\| \\
&\quad
+
\|f(s^*(t),\pi^*(s(t))) - f(s^*(t),\pi^*(s^*(t)))\| \\
&\le
L_a\|e(s(t))\|
+
L_s\|s(t)-s^*(t)\|
+
L_a\|\pi^*(s(t))-\pi^*(s^*(t))\| \\
&\le
L_a\|e(s(t))\|
+
L_s\|\delta(t)\|
+
L_aL_{\pi^*}\|\delta(t)\| \\
&=
L_a\|e(s(t))\|
+
\big(L_s+L_aL_{\pi^*}\big)\|\delta(t)\|.
\end{aligned}
\]
Using the bound in Eq.~(\ref{eq:policy_error_bound_app}), we get:
\[
\|\dot{\delta}(t)\|
\le
L_a \|e(s^*(t))\|
+
\big(L_s + L_a L_{\pi^*} + L_a L_e\big)\|\delta(t)\|.
\]
Let $\beta = L_s + L_a L_{\pi^*} + L_a L_e$. Then:
\[
\|\dot{\delta}(t)\|
\le
\beta \|\delta(t)\| + L_a \|e(s^*(t))\|.
\]

Let $u(t)=\|\delta(t)\|$. Under Assumptions~\ref{ass:domain} and~\ref{ass:lipschitz_dynamics}, the trajectories $s(t)$ and $s^*(t)$ are absolutely continuous. Hence, $\delta(t)=s(t)-s^*(t)$ is absolutely continuous. Since the norm is Lipschitz continuous, $u(t)$ is also absolutely continuous. Moreover, for almost all $t\in(0,T)$, we have
\[
\dot u(t)\le \|\dot{\delta}(t)\|.
\]
Therefore,
\begin{equation}
\label{uledelta}
\dot u(t)
\le
\beta u(t) + L_a\|e(s^*(t))\|,
\qquad \text{for a.e. } t\in(0,T).
\end{equation}
For Eq.~(\ref{uledelta}), applying Lemma~\ref{lem:gronwall} with $a(t)=\beta$ and $b(t)=L_a\|e(s^*(t))\|$ yields
\[
\|\delta(t)\|
\le
e^{\beta t}\|\delta(0)\|
+
L_a\int_0^t e^{\beta(t-\tau)}\|e(s^*(\tau))\|\, d\tau.
\]

This can be further bounded as:
\[
\|\delta(t)\|
\le
e^{\beta t}\|\delta(0)\|
+
\frac{L_a}{\beta}
\big(e^{\beta t}-1\big)
\max_{0 \le \tau \le t}\|e(s^*(\tau))\|.
\]

Assuming $\delta(0)=0$, we get:
\[
\|\delta(t)\|
\le
\frac{L_a}{\beta}\big(e^{\beta t}-1\big)
\max_{0\le \tau\le t}\|e(s^*(\tau))\|.
\]
This completes the proof.
\end{proof}
    
\section{Experimental Details}
\label{app:exp_details}
We evaluate our method on two widely used offline RL benchmarks: OGBench~\citep{ogbench_park2025} and D4RL~\citep{fu2020d4rl}. OGBench provides high-dimensional, long-horizon, and structurally diverse tasks that go well beyond standard locomotion, while D4RL offers a well-established testbed of MuJoCo locomotion tasks with datasets of varying behavioral quality. Together, they cover a broad spectrum of state and action dimensionalities, task structures, and dataset qualities, providing a comprehensive evaluation of our scheme. All methods are trained on the offline datasets only, with no online interaction.

\subsection{Benchmarks and Tasks}
\label{app:benchmark_details}
% todo: citation space 

\textbf{OGBench.}\quad
Since OGBench was originally designed for goal-conditioned RL, we adopt its reward-based single-task variants (denoted by the \texttt{-singletask} suffix) to benchmark standard reward-maximizing offline RL methods. In total, we evaluate on $50$ state-based tasks drawn from $10$ environments ($5$ tasks per environment), covering long-horizon locomotion (\texttt{antmaze}, \texttt{humanoidmaze}), object interaction (\texttt{antsoccer}), and robotic manipulation (\texttt{cube}, \texttt{scene}, \texttt{puzzle}), where \texttt{scene} additionally requires long-horizon multi-object manipulation and \texttt{puzzle} involves more combinatorial reasoning.

These tasks pose substantial challenges in high-dimensional continuous control, with state dimensions ranging from $28$ to $83$ and action dimensions from $5$ to $21$. For instance, \texttt{humanoidmaze} has a $69$-dimensional state space and a $21$-dimensional action space, while \texttt{cube-double}, \texttt{scene}, and \texttt{puzzle-4x4} have $37$-, $40$-, and $83$-dimensional state spaces, respectively. We use the standard reward-based datasets provided by OGBench, which are highly suboptimal as they aggregate trajectories from policies performing diverse random tasks (e.g., reaching different goals or manipulating different objects). This makes it particularly difficult for offline RL methods to extract task-relevant behaviors from such heterogeneous data.

\textbf{D4RL.}\quad
We further evaluate our method on the MuJoCo locomotion tasks from D4RL, including \texttt{halfcheetah}, \texttt{hopper}, and \texttt{walker2d}, with the standard \texttt{medium}, \texttt{medium-replay}, \texttt{medium-expert}, and \texttt{expert} dataset types collected from behavior policies of varying quality. The full-size datasets contain about $1$M transitions for the \texttt{medium} and \texttt{expert} settings, $2$M transitions for \texttt{medium-expert}, and around $200$k--$400$k transitions for \texttt{medium-replay}. Beyond the standard full-data setting, we additionally consider a small-dataset setting to study the data efficiency of our scheme. Specifically, for each task we first split the original dataset into individual trajectories, and then sample a small set of complete trajectories until the accumulated transitions reach $10$k, rather than randomly sampling individual transitions. This trajectory-level sampling preserves the temporal continuity and dynamical consistency of the original demonstrations, which is essential for our first-order action supervision since it relies on consecutive transitions $(s,s',a,a')$ to compute the local state and action derivatives $(\dot{s},\dot{a})$. To ensure a fair comparison, all methods use the same subsampled dataset under the same random seed.

\subsection{Methods and Hyperparameters}
\label{app:imp_details}
We instantiate our first-order action supervision on four representative offline RL algorithms covering deterministic, stochastic, and flow-based policies: TD3+BC~\citep{fujimoto2021minimalist}, ReBRAC~\citep{tarasov2023revisiting}, IQL~\citep{iql}, and IFQL~\citep{Park2025FlowQ}. For a fair comparison, all algorithms are implemented on top of the FQL~\citep{Park2025FlowQ} codebase and share the same network architectures and default hyperparameters unless otherwise specified. For OGBench tasks, we follow the hyperparameter settings reported in FQL~\citep{Park2025FlowQ} for ReBRAC, IQL, and IFQL. After integrating our scheme, we keep all baseline-specific hyperparameters unchanged, and only specify the two additional weights $\lambda_Q$ (the value-side regularization weight) and $\lambda_{\dot{a}}$ (the policy-side first-order supervision weight) introduced by our scheme. The basic shared hyperparameters are summarized in Table~\ref{tab:hyperparameters}.

\textbf{TD3+BC.}\quad
We re-implement TD3+BC on top of the FQL~\citep{Park2025FlowQ} codebase. On OGBench, we sweep the Q-value scaling coefficient $\alpha$ over $\{1, 3, 10, 30, 100, 300\}$ and select the best value per environment. Intuitively, a smaller $\alpha$ is preferable in complex environments where the Q-function is hard to learn reliably, while a larger $\alpha$ works better on simpler ones where the learned Q-function offers stronger policy improvement signals. We also observe Q-learning instability on several challenging OGBench tasks, and apply layer normalization to the Q-network to mitigate this issue. On D4RL, we follow the hyperparameter settings of the original TD3+BC paper~\citep{fujimoto2021minimalist}, and use the default values for all other hyperparameters (e.g., noise standard deviation, noise clipping threshold). For our scheme, we set $\lambda_Q = 0.1$ on all OGBench tasks, $\lambda_{\dot{a}} = 1.0$ on \texttt{antmaze} and $\lambda_{\dot{a}} = 0.1$ on all other OGBench tasks; on D4RL, we set $\lambda_Q = 0.1$ and $\lambda_{\dot{a}} = 1.0$.

\textbf{ReBRAC.}\quad
We follow the hyperparameter settings reported in the original ReBRAC paper~\citep{tarasov2023revisiting} on D4RL, including the actor and critic BC coefficients $(\alpha_a, \alpha_c)$. For our scheme, we set $\lambda_Q = 0.1$ on all OGBench tasks, $\lambda_{\dot{a}} = 1.0\,\alpha_a$ on \texttt{antmaze}, \texttt{humanoidmaze}, and \texttt{antsoccer}, and $\lambda_{\dot{a}} = 0.1\,\alpha_a$ on \texttt{cube}, \texttt{scene}, and \texttt{puzzle}; on D4RL, we set $\lambda_Q = 0.1$ and $\lambda_{\dot{a}} = 1.0\,\alpha_a$.

\textbf{IQL.}\quad
We follow the hyperparameter settings reported in the original IQL paper~\citep{iql} on D4RL, including the expectile value and the AWR inverse temperature. For our scheme on OGBench, we set $\lambda_Q = 0.5$ on \texttt{scene} and $\lambda_Q = 0.1$ on all other tasks, and $\lambda_{\dot{a}} = 1.0$ on \texttt{antmaze} and \texttt{scene} and $\lambda_{\dot{a}} = 0.1$ on all other tasks; on D4RL, we set $\lambda_Q = 0.1$ and $\lambda_{\dot{a}} = 1.0$.

\textbf{IFQL.}\quad
For IFQL, we use $N = 32$ action samples on D4RL tasks, and follow all other settings from FQL~\citep{Park2025FlowQ}. For our scheme, we set $\lambda_Q = 1.0$ and $\lambda_{\dot{a}} = 1.0$ on all OGBench and D4RL tasks.
\vspace{-1pt}
\begin{table}[H]
\centering
\caption{Shared hyperparameters used across all four offline RL algorithms.}
\label{tab:hyperparameters}
\footnotesize
\begin{tabular}{lp{0.55\textwidth}}
\toprule
\textbf{Hyperparameter} & \textbf{Value} \\
\midrule
Learning rate & $3 \times 10^{-4}$ \\
Optimizer & Adam~\citep{kingma2014adam} \\
Gradient steps & $1{,}000{,}000$ \\
Minibatch size & $256$ \\
MLP dimensions & $[512, 512, 512, 512]$ (default); $[256, 256, 256, 256]$ (IQL) \\
Nonlinearity & GELU~\citep{hendrycks2016gaussian} \\
Target network smoothing coefficient & $0.005$ \\
Discount factor $\gamma$ & $0.99$ (default); $0.995$ (\texttt{antmaze-giant}, \texttt{humanoidmaze}, \texttt{antsoccer}) \\
Temporal correlation coefficient $\lambda$ & $1.0$ \\
\bottomrule
\end{tabular}
\end{table}

\section{Broader Impact} 
\label{app:broadimp}
Modern learning-based decision-making methods, such as imitation learning (IL) and reinforcement learning (RL), often suffer from control instability and poor robustness when deployed in real-world systems, due to their sole reliance on zeroth-order action supervision. To address this, we propose a lightweight and plug-and-play learning scheme that simultaneously supervises zeroth- and first-order actions, transforming any standard policy model into a more powerful policy class without structural modifications. Our method substantially improves action smoothness, control stability, and out-of-distribution generalization, especially in challenging low-data regimes, making it well-suited for safety-critical domains such as robotics, drone control, and autonomous driving, where high-frequency action fluctuations can cause mechanical wear, energy waste, and safety risks. By bridging classical control principles with modern policy learning, our work offers a promising pathway toward more reliable and deployable learning-based controllers. A potential limitation is that, like other offline methods, the learned policy may still inherit unsafe or suboptimal behaviors present in the historical dataset.
\newpage
\end{document}